\documentclass{article} %
\usepackage[preprint,nonatbib]{neurips_2026}

\usepackage{url}            %
\usepackage{booktabs}       %
\usepackage{amsfonts}       %
\usepackage{nicefrac}       %
\usepackage{microtype}      %
\usepackage[dvipsnames]{xcolor}         %

\usepackage{microtype}
\usepackage{graphicx}
\usepackage{subcaption}
\usepackage{authblk}

\usepackage[numbers,sort,compress]{natbib}
\usepackage{times}
\usepackage{lipsum}
\usepackage{algorithm}
\usepackage{algorithmic}
\usepackage{siunitx}

\usepackage{eqparbox}

\usepackage{titlesec}
\titlespacing{\section}{0pt}{7pt}{4pt}
\titlespacing{\subsection}{0pt}{5pt}{3pt}

\usepackage[utf8]{inputenc}
\usepackage[T1]{fontenc}    %
\usepackage[colorlinks=true, linkcolor=BrickRed, urlcolor=Blue, citecolor=Blue, anchorcolor=blue, backref=page]{hyperref}
\renewcommand*{\backref}[1]{}
\renewcommand*{\backrefalt}[4]{%
    \ifcase #1%
          \or [Cited on p.~#2.]%
          \else [Cited on p.~#2.]%
    \fi%
    }
\usepackage{wrapfig}
\usepackage{amsmath}
\usepackage{nccmath}
\usepackage{amsthm}
\usepackage{amssymb}
\usepackage{bm}
\usepackage{url}            %
\usepackage{booktabs}
\usepackage{multirow}
\usepackage{amsfonts} 
\usepackage{comment}
\usepackage{tikz}
\usetikzlibrary{shapes.geometric, arrows, shadows, bayesnet}
\usepackage{subcaption}
\usepackage{xcolor}

\usepackage{enumitem}
\usepackage{tabularx}
\usepackage{soul}
\usepackage{enumitem}

\numberwithin{equation}{section}
\renewcommand{\mathbf}[1]{{\bm{#1}}}
\newcommand{\ab}{\mathbf{a}}
\newcommand{\bb}{\mathbf{b}}
\newcommand{\cbb}{\mathbf{c}}

\newcommand{\fb}{\mathbf{f}}
\newcommand{\gb}{\mathbf{g}}
\newcommand{\hb}{\mathbf{h}}

\newcommand{\wb}{\mathbf{w}}
\newcommand{\xb}{\mathbf{x}}

\newcommand{\zb}{\mathbf{z}}

\newcommand{\Ab}{\mathbf{A}}
\newcommand{\Bb}{\mathbf{B}}
\newcommand{\Cb}{\mathbf{C}}

\newcommand{\Hb}{\mathbf{H}}
\newcommand{\Ib}{\mathbf{I}}
\newcommand{\Jb}{\mathbf{J}}
\newcommand{\Kb}{\mathbf{K}}

\newcommand{\Mb}{\mathbf{M}}

\newcommand{\Ob}{\mathbf{O}}

\newcommand{\Vb}{\mathbf{V}}
\newcommand{\Wb}{\mathbf{W}}
\newcommand{\Xb}{\mathbf{X}}
\newcommand{\Yb}{\mathbf{Y}}
\newcommand{\Zb}{\mathbf{Z}}

\newcommand{\Acal}{\mathcal{A}}

\newcommand{\Dcal}{\mathcal{D}}

\newcommand{\Lcal}{\mathcal{L}}

\newcommand{\Ncal}{\mathcal{N}}

\newcommand{\Pcal}{\mathcal{P}}

\newcommand{\Rcal}{\mathcal{R}}

\newcommand{\EE}{\mathbb{E}} %
\newcommand{\NN}{\mathbb{N}} %
\newcommand{\PP}{\mathbb{P}} %
\newcommand{\QQ}{\mathbb{Q}} %
\newcommand{\RR}{\mathbb{R}} %
\newcommand{\ZZ}{\mathbb{Z}} %

\newcommand*{\alphab}{\bm{\alpha}}

\newcommand*{\Sigmab}{\bm{\Sigma}}

\newcommand*{\thetab}{{\bm{\theta}}}

\newcommand*{\epsilonb}{{\bm{\varepsilon}}}
\newcommand*{\etab}{\bm{\eta}}
\newcommand*{\xib}{\bm{\xi}}

\newcommand*{\mub}{\bm{\mu}}

\newcommand*{\Phib}{\bm{\Phi}}
\newcommand*{\phib}{{\bm{\phi}}}

\usepackage{amsthm,thmtools,thm-restate}
\usepackage{cleveref}
\RequirePackage{amsmath}
\RequirePackage{amssymb}
\RequirePackage{mathtools}
\ifx\proof\undefined 
\RequirePackage{amsthm}
\fi
\crefname{figure}{Fig.}{Figs.}
\crefname{algorithm}{Alg.}{Algs.}
\crefname{definition}{Defn.}{Defns.}
\crefname{corollary}{Cor.}{Cors.}
\crefname{proposition}{Prop.}{Props.}
\crefname{theorem}{Thm.}{Thms.}
\crefname{remark}{Remark}{Remarks}
\crefname{principle}{Principle}{Principles}
\crefname{lemma}{Lemma}{Lemmata}
\crefname{claim}{Claim}{Claims}
\crefname{table}{Tab.}{Tabs.}
\crefname{section}{\S}{\S\S}
\crefname{subsection}{\S}{\S\S}
\crefname{subsubsection}{\S}{\S\S}
\crefname{assumption}{Asm.}{Asms.}
\crefname{appendix}{Appx.}{Appx.}
\crefname{equation}{Eq.}{Eqs.}
\crefname{example}{Ex.}{Exs.}
\newcommand{\dimPhi}{d_\Phi}

\ifx\BlackBox\undefined
\newcommand{\BlackBox}{\rule{1.5ex}{1.5ex}}  %
\fi

\ifx\QED\undefined
\def\QED{~\rule[-1pt]{5pt}{5pt}\par\medskip}
\fi

\ifx\proof\undefined
\newenvironment{proof}{\par\noindent{\bf Proof\ }}{\hfill\BlackBox\\[2mm]}
\fi
\ifx\proofsketch\undefined
\newenvironment{proofsketch}{\par\noindent{\bf Proof Sketch\ }}{\hfill\BlackBox\\[2mm]}
\fi

\theoremstyle{plain} %
\ifx\theorem\undefined
\newtheorem{theorem}{Theorem}
\numberwithin{theorem}{section}
\fi
\ifx\property\undefined
\newtheorem{property}{Property}
\fi
\ifx\corollary\undefined
\newtheorem{corollary}[theorem]{Corollary}
\fi
\ifx\lemma\undefined
\newtheorem{lemma}[theorem]{Lemma}
\fi
\ifx\proposition\undefined
\newtheorem{proposition}[theorem]{Proposition}
\fi
\ifx\assum\undefined

\fi

\theoremstyle{definition} %
\ifx\definition\undefined
\newtheorem{definition}[theorem]{Definition}
\fi
\ifx\assum\undefined

\fi

\theoremstyle{remark} %
\ifx\remark\undefined
\newtheorem{remark}[theorem]{Remark}
\fi
\ifx\example\undefined
\newtheorem{example}[theorem]{Example}
\fi
\ifx\lemma\undefined
\newtheorem{lemma}[theorem]{Lemma}
\fi
\ifx\conjecture\undefined
\newtheorem{conjecture}[theorem]{Conjecture}
\fi
\ifx\fact\undefined
\newtheorem{fact}[theorem]{Fact}
\fi
\ifx\claim\undefined
\newtheorem{claim}[theorem]{Claim}
\fi
\ifx\assum\undefined

\fi

\DeclareMathOperator*{\argmin}{arg\,min}

\newcommand\abs[1]{\left|#1\right|}

  \usepackage{colortbl}

  \newcommand{\norm}[1]{\left\lVert#1\right\rVert}

\newcommand{\Zbh}{\widehat{\mathbf{Z}}}
\newcommand{\Xbh}{\widehat{\mathbf{X}}}
\newcommand{\zbh}{\widehat{\mathbf{z}}}
\newcommand{\xbh}{\widehat{\mathbf{x}}}
\newcommand{\enc}{\gb}
\newcommand{\dec}{\fb}
\newcommand{\ench}{\widehat{\enc}}
\newcommand{\dech}{\widehat{\dec}}

\newcommand{\Wbh}{\widehat{\Wb}}
\newcommand{\PPh}{\widehat{\PP}}
\newcommand{\CRPS}{\operatorname{CRPS}}
\newcommand{\ES}{\operatorname{ES}}
\newcommand{\MMD}{\operatorname{MMD}}
\newcommand{\ED}{\operatorname{ED}}
\newcommand{\pert}{\mathrm{pert}}
\newcommand{\base}{\mathrm{base}}
\newcommand{\iidsim}{\overset{\mathrm{iid}}{\sim}}
\newcommand{\dimZ}{{d_Z}}
\newcommand{\dimX}{{d_X}} %
\newcommand{\dimeps}{{d_\epsilon}}

\newcommand{\fbt}{\widetilde{\fb}}

\newcommand{\dect}{\fbt}
\newcommand{\Zbt}{\widetilde{\Zb}}

\newcommand{\Wbt}{\widetilde{\Wb}}
\newcommand{\Cbt}{\widetilde{\Cb}}
\newcommand{\cbt}{\widetilde{\cbb}}
\newcommand{\PPt}{\widetilde{\PP}}
\newcommand{\pt}{\widetilde{p}}
\newcommand{\rank}{\mathrm{rank}}
\newcommand{\te}{\mathrm{test}}
  \newcommand{\pluseq}{\mathrel{+}=}
\usepackage{setspace}

\usepackage[toc,page,header]{appendix}
\usepackage{minitoc}
\newcommand{\changelinkcolor}[1]{\hypersetup{linkcolor=#1}}  

\AtBeginDocument{%
  
}

\newcommand\jonas[1]{{{\color{blue} Jonas: #1}}}
\newcommand\rmj[1]{{\color{gray}\tiny Rm(J): #1}}
\newcommand{\jvk}[1]{\textcolor{Green}{{#1}}}
\newcommand\xinwei[1]{{{\color{cyan} Xinwei: #1}}}
\renewcommand\jonas[1]{}
\renewcommand\rmj[1]{}
\renewcommand{\jvk}[1]{}
\renewcommand\xinwei[1]{}

\title{Perturbation Modeling with Additive Latent Shifts:\\ Identifiability and Extrapolation Guarantees}
\title{Perturbation Extrapolation via Additive Latent Shifts}
\title{Perturbation Extrapolation Guarantees\\ Under the Additive Latent Shift Assumption}
\title{Extrapolation Guarantees for Perturbation Modeling\\ Under the Additive Latent Shift Assumption}

\newcommand\blfootnote[1]{%
  \begingroup
\renewcommand\thefootnote{}\footnote{#1}%
  \addtocounter{footnote}{-1}%
  \endgroup
}

\author[1]{Julius von K{\"u}gelgen}
\author[2]{Jakob Ketterer}
\author[1]{Michael Vollenweider}
\author[3]{\authorcr Michael Scholkemper}
\author[2,4]{Xinwei Shen}
\author[2]{Nicolai Meinshausen}
\author[1]{Jonas Peters}
\affil[1]{Seminar for Statistics, ETH Zurich, Switzerland}
\affil[2]{ETH Zurich, Switzerland}
\affil[3]{DZNE, Bonn, Germany}
\affil[4]{Department of Statistics, University of Washington, Seattle, USA}

\begin{document}
\doparttoc
\faketableofcontents%
\maketitle
\blfootnote{\looseness-1 Correspondence to: \href{mailto:vjulius@ethz.ch}{\text{vjulius@ethz.ch}}. JK, XS, and NM are not with ETH anymore, but mostly contributed while at ETH.}
\vspace{-3.em}
\begin{abstract}
\vspace{-0.5em}
We  consider the problem of modeling the effects of perturbations like gene knockouts on measurements such as single-cell RNA counts. 
Given data for some perturbations, we aim to predict the distribution of measurements for new combinations of perturbations.
To address this challenging extrapolation task, we posit that perturbations act additively in a suitable, unknown embedding space. 
We formulate the data-generating process as a latent variable model, in which perturbations amount to mean shifts in latent space and can be combined additively.
We then prove that, given sufficiently diverse training perturbations, the representation and perturbation effects are identifiable up to orthogonal transformation and use this to derive extrapolation guarantees for unseen perturbations that can be expressed as linear combinations of seen ones.
To estimate the model from data, we propose the perturbation distribution autoencoder (PDAE), which is trained by maximizing the distributional similarity between true and simulated perturbation distributions. 
The trained model can then be used to predict  previously unseen perturbation distributions.
In support of our theoretical results, we demonstrate through simulations that PDAE can accurately predict the effects of unseen but identifiable perturbations, and showcase the method on combinatorial gene perturbation data.
\end{abstract}
\vspace{-0.25em}

\section{Introduction}

\looseness-1
Due to technological progress, large-scale perturbation data is becoming increasingly abundant across several scientific fields.
However, the number of possible combinations of perturbations is often exponentially large, rendering exhaustive experimentation infeasible. Observations are thus typically only available for a subset of perturbations of interest.
To identify promising candidates and thereby inform future data collection, we need 
models capable of extrapolating to unseen perturbations.

\textbf{Motivating example.}
In single-cell biology, advances in gene editing and sequencing~\citep{wang2009rna,jinek2012programmable,dixit2016perturb} %
have enabled the collection of vast %
databases for various gene~\citep{norman2019exploring,wessels2023efficient} and drug~\citep{zhang2025tahoe} perturbations.
These data sets contain expression levels of several thousands of genes under selected single and double gene knockouts or drug dosages across a large number of cells. 
Several studies advocate for the use of machine learning to model such biological perturbations~\citep{lopez2018deep,lopez2023learning,hugi2025perturbation}, hoping to generalize to
unseen combinations like new multi-gene knockdowns or dosage combinations~\citep{lotfollahi2023predicting,bereket2023modeling,he2025morph},
new cell types~\citep{lotfollahi2019scgen,bunne2023learning,adduri2025predicting}, or
entirely new perturbations based on prior knowledge~\citep{yu2022perturbnet,hetzel2022predicting,qi2024predicting,roohani2024predicting,kamimoto2023dissecting,he2025morph}.

\looseness-1
\looseness-1
A common theme in prior work is the use of (nonlinear) representation learning techniques such as autoencoders~\citep{rumelhart1986learning,hinton2006reducing,kingma2013auto} %
to embed observations in a lower-dimensional latent space, in which the effects of perturbations may 
take on a simpler form.
However, most existing approaches are purely empirical and lack 
theoretical underpinning; 
their capabilities and limitations %
thus remain poorly understood. 
To build the next generation of perturbation models,
it is crucial to gain a better understanding of the underlying principles and  suitable assumptions for the fundamentally difficult task of extrapolation.

\begin{figure}[t]
\vspace{-1.75em}
    \centering
    \begin{subfigure}[b]{0.55\textwidth}
    \centering    \includegraphics[width=\textwidth]{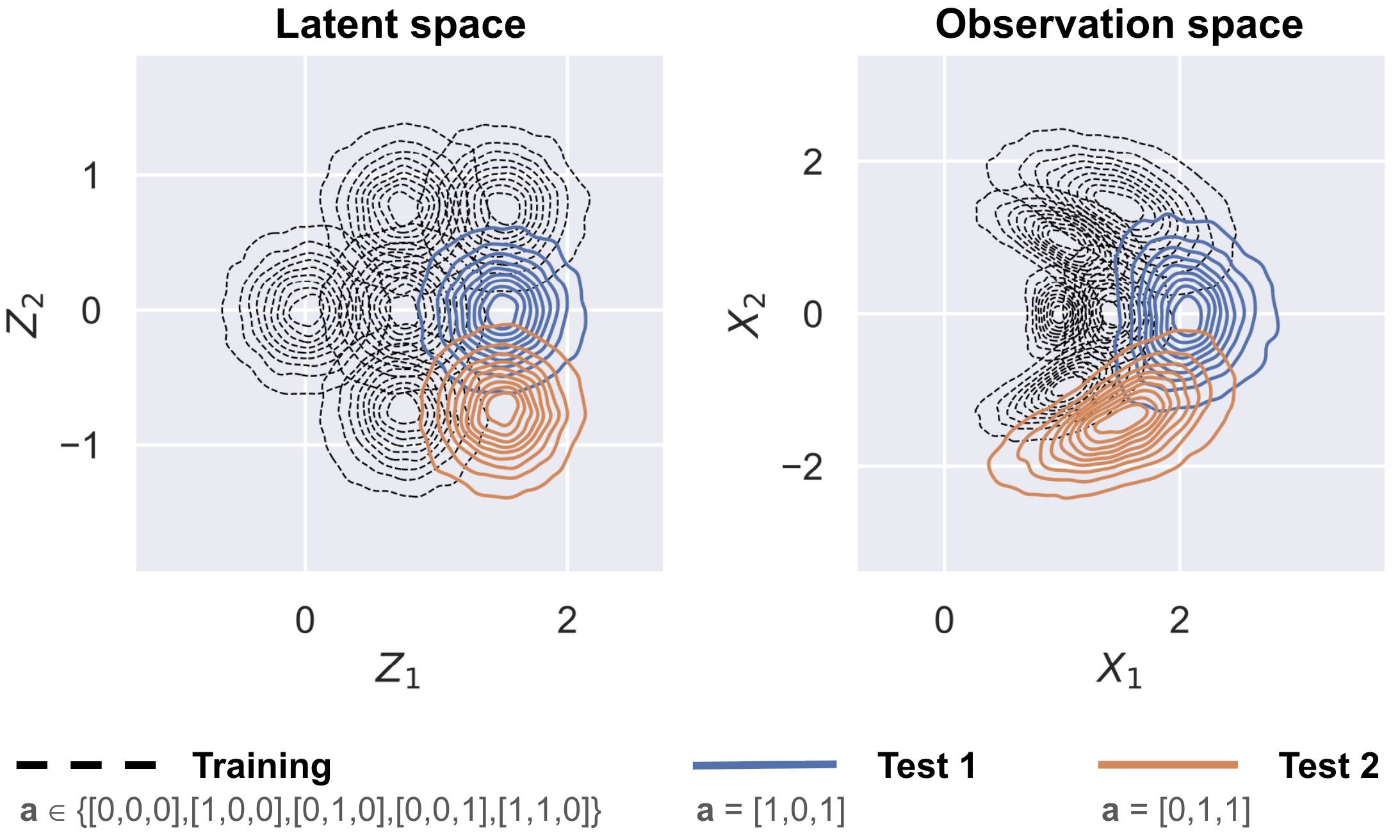}
    \caption{Simulated multi-domain perturbation data}
    \label{subfig:simulated_data}
    \end{subfigure}\hfill
    \begin{subfigure}[b]{0.4\textwidth}
    \centering    
    \begin{tikzpicture}[scale=0.25]
    \centering
    \newcommand{\xshift}{-5.em}
    \newcommand{\yshift}{-4.em}
    \newcommand{\nodewidth}{2.5em}
        \node (Z_base) [thick, latent, minimum width = \nodewidth] {\small $\Zb^\base_{e,i}$};
        \node (Z_pert) [thick, latent, yshift = \yshift, minimum width = \nodewidth] {\small $\Zb^\pert_{e,i}$};
        \node (Y) [thick, obs, xshift=\xshift, minimum width = \nodewidth] {\small $\ab_e$};
        \node (eps) [thick, latent, xshift=-\xshift, yshift=\yshift, minimum width = \nodewidth] {\small $\epsilonb_{e,i}$};
        \node (X) [thick, obs, yshift=2*\yshift, minimum width = \nodewidth] {\small $\Xb_{e,i}$};
        \edge[-stealth, thick]{Z_base, Y}{Z_pert};
        \edge[-stealth, thick]{Z_pert, eps}{X};
        \tikzset{plate caption/.style={caption, node distance=0, inner sep=-10pt, below left=-5pt and 5pt of #1.south east, text depth=0.2em}}
        \plate [inner sep=0.5em]{i} {(Z_base) (Z_pert) (eps) (X)} {$N_e$};
       \tikzset{plate caption/.style={caption, node distance=0, inner sep=-1pt, below left=-5pt and -25pt of #1.south west,text depth=0.2em}}
        \plate[inner sep=1em] {e} {(Y) (Z_base) (Z_pert) (eps) (X)} {$M+1$}; 
    \end{tikzpicture}
    \caption{Graphical model representation
    }
    \label{subfig:graphical_model}
    \end{subfigure}
    \caption{
    \small 
    \looseness-1 
    \textbf{Task description and example following the assumed data generating process.}
    \textbf{(a)}~During training, we are given {$M\!=\!5$} training
    data sets in observation space (right, contour plots in grey), each of which is generated under a known combination of {$K\!=\!3$} elementary  
    perturbations. The corresponding (training) perturbation labels~$\ab_e$ are shown below the plots. At test time, we are given a new perturbation label and the task is to predict the corresponding distribution in observation space (right, blue and orange).
    We tackle this task by assuming that the effect of perturbations is linear additive in a suitable latent space (left).
    Both plots show kernel density estimates of the distributions.
    \textbf{(b)}~For each experiment or environment $e\in[M]_0$, the corresponding dataset comprises a perturbation label~$\ab_e$ and $N_e$ observations~$\bm{x}_{ei}$. Perturbations are assumed to act as mean-shifts on a latent basal state, $\zb^\text{pert}_{ei}=\zb^\text{base}_{ei}+\Wb\ab_e$ for an unknown perturbation matrix $\Wb$. An unknown nonlinear (stochastic) decoder with noise~$\bm\varepsilon_{ei}$ then yields the observed $\xb_{ei}=\dec(\zb^\pert_{ei},\bm\varepsilon_{ei})$. 
    In the example in (a), $\dec$ is constant in its second argument.
    Shaded and white nodes indicate observed and unobserved/latent variables, respectively.}
    \label{fig:data_generating_process}
    \vspace{-0.5em}
\end{figure}

\textbf{Overview and main contributions.} 
 \looseness-1 
In this work, we present a theoretically-grounded perspective on extrapolation to unseen perturbations. 
Given the often unpaired nature of the available non-i.i.d.\ data %
(e.g., each cell can only be measured once since measurements are destructive), %
we consider the task of predicting population-level perturbation effects, which we formalize as a distributional regression problem~(\cref{sec:problem_setting}).
We then postulate a generative model~(\cref{sec:model}) in which, similar to prior works~\citep{lotfollahi2023predicting,bereket2023modeling}, perturbations act as additive mean shifts in a suitable latent space~(\cref{fig:data_generating_process}).

For the case of a deterministic decoder, we analyze this model class theoretically~(\cref{sec:theory}), 
proving
that, under mild assumptions, the latent representation and the relative training perturbation effects are identifiable up to orthogonal transformation if the available  environments are sufficiently diverse~(\cref{thm:affine_identifiability_gaussian}).
This result implies
extrapolation guarantees~(\cref{thm:extrapolation}) for unseen perturbations that can be expressed as linear combinations of training perturbations~(\cref{fig:span}).

For estimation, we propose %
the perturbation distribution autoencoder (PDAE; \cref{fig:architecture}), which 
uses the energy score~\citep{gneiting2007strictly} to maximize distributional similarity between simulated and ground-truth perturbation data~(\cref{sec:method}). 
Through controlled simulations~(\cref{sec:experiments}), we verify that, contrary to baselines, PDAE can accurately capture unseen identifiable perturbation distributions~(\cref{fig:qualitative_results_main}).
In a case study
on combinatorial gene perturbation data~\citep{norman2019exploring,wessels2023efficient}~(\cref{sec:real_world_case_study}), we find PDAE to perform better than prior deep learning approaches and on par with its linear counterpart~(\cref{fig:results_norman}), consistent with recent findings~\citep{ahlmann2025deep}.

\textbf{Notation.}
\looseness-1 
We write scalars as $a$, column-vectors as $\ab$, and matrices as $\Ab$. 
We use uppercase for random variables and lowercase for their realizations.
The pushforward of a distribution $\PP$ by a measurable function $f$ is denoted by~$f_\#\PP$. 
The Euclidean ($\ell_2$) norm is denoted by $\norm{\cdot}$ and the set of orthonormal $d\times d$ matrices by $O(d)$.
For $n\in\NN_{\geq 1}$, we write $[n]:=\{1, ..., n\}$ and $[n]_0:=[n]\cup\{0\}$.

\section{Problem setting: Distributional perturbation extrapolation}
\label{sec:problem_setting}
\textbf{Observables.}
Let $\xb\in\RR^{\dimX}$ be 
an \textit{observation}
(e.g., omics data)
that 
is
obtained under 
one
of several possible
experimental conditions, and let $\Xb$ be the corresponding random vector. %
We model these conditions as combinations of $K$ elementary
perturbations, each of which 
we assume 
can be encoded by a real number. 
Further, let $\ab\in\RR^K$ be 
a vector of 
\textit{perturbation labels%
} that indicate%
s
if, or how much of, each perturbation was applied before collecting~$\xb$.

\textbf{Data.} We have access to $M+1$ experimental datasets $\Dcal_0, \Dcal_1, ..., \Dcal_M$, each comprising a sample of $N_e$~observations~$\xb$ and a perturbation label $\ab$, i.e., for all experiments or environments~$e\in %
[M]_0$,
\begin{equation}
\label{eq:data}
    \Dcal_e = \big((\xb_{e,i})_{i=1}^{N_e}, \ab_e\big).
\end{equation}
\looseness-1 For $e\in[M]_0$, we assume $\xb_{e,1}, ..., \xb_{e,N_e}$ are i.i.d.\ realizations from an underlying distribution~$\PP_{\Xb|\ab_e}$.

\textbf{Task.} 
\looseness-1 
We aim
to predict the effect of new perturbations ${\ab_\te\not\in\{\ab_0,\ab_1, ..., \ab_M\}}$ 
without observing any data from this condition (%
``zero-shot''). 
In particular, we are interested in the distribution over observations resulting from $\ab_\te$. That is, 
we aim to leverage the training domains~\eqref{eq:data} to learn a map 
\begin{equation}
\label{eq:distributional_regression_formulation}
\ab\mapsto \PP_{\Xb|\ab},
\end{equation}
\looseness-1
which extrapolates beyond the training support;
predictions 
should
remain reliable for new~$\ab_\te$.
\begin{example}%
\label{ex:gene_perturbations}
\looseness-1 
Gene knockouts can be represented by (sparse) binary ${\ab\in\{0,1\}^K}$, where $K$ is the number of potential targets and $a_k=1$ iff.\ target~$k$ is knocked out. 
We observe control cells from ${\ab_0\!=\!(0,0,0)^\top}$ and three single-gene knockouts $\ab_1\!=\!(1,0,0)^\top$, $\ab_2\!=\!(0,1,0)^\top$, and $\ab_3\!=\!(0,0,1)^\top$, and want to predict the distribution under double knockouts $\ab^1_{\text{test}}=(1,1,0)^\top$ and $\ab^2_{\text{test}}=(1,0,1)^\top$.

\end{example}
\textbf{Distributional vs mean prediction.} 
Since~\eqref{eq:distributional_regression_formulation} targets the full conditional distribution---rather than, say, the conditional mean $\EE[\Xb|\ab]$---it constitutes a (multi-variate) distributional regression task~\citep{koenker1978regression}, a.k.a.\ probabilistic forecasting~\citep{gneiting2007strictly} or conditional generative modeling~\citep{sohn2015learning}.
Therefore, we refer to the problem setting addressed in this paper 
as \textit{distributional} perturbation extrapolation. 

\looseness-1 
\textbf{On extrapolation.}
Formally, extrapolation means that the 
value  of the function in~\eqref{eq:distributional_regression_formulation} at 
$\ab_\te$ 
is determined by its values
on the training support $\{\ab_0,\ab_1, ..., \ab_M\}$.
Intuitively, for this to be feasible without imposing very restrictive assumptions on the form of the mapping in~\eqref{eq:distributional_regression_formulation}, $\ab_\te$ should be somehow related to the training perturbations $\ab_e$. %
For example, given data resulting from individual perturbations, predict the effects of combinations thereof.
This type of extrapolation to new combinations of inputs is also called compositional generalization~\citep{lake2017building,goyal2022inductive}.
It is known to be challenging~\citep{schott2021visual,montero2021the,montero2022lost} and requires assumptions that sufficiently constrain the model class%
~\citep{wiedemer2024provable,wiedemer2024compositional,lachapelle2024additive,brady2023provably,brady2025interaction,dong2022first,lippl2024does}.

\section{Model: Perturbations as additive mean shifts in latent space}
\label{sec:model}
We now specify a generative process for the observed data in~\eqref{eq:data}. 
In so doing, we aim to strike a balance between imposing sufficient structure on~\eqref{eq:distributional_regression_formulation} to facilitate extrapolation, while remaining flexible enough to model the complicated, nonlinear effects which perturbations may have on the distribution of observations. 
\looseness-1 
Similar to other approaches
\citep[e.g.,][]{lotfollahi2023predicting,bereket2023modeling}, we model the effect of perturbations as mean shifts in a latent space with \text{perturbation-relevant latent variables} $\zb\in\RR^{\dimZ}$, 
which are related to the observations $\xb$ via a nonlinear (stochastic) \textit{mixing function} or \textit{generator}~$\dec$.
The full generative process amounts to a hierarchical latent variable model, 
which additionally contains \textit{noise variables}~$\epsilonb$ that capture other variation underlying the observations~$\xb$, 
and which is represented as a graphical model in~\cref{subfig:graphical_model}. 
 Specifically, we posit for all $e\in[M]_0\cup \{\te\}$ and all $i\in [N_e]$:
\begin{align}
    \Zb^\base_{e,i}\sim \PP_\Zb,
    \qquad 
    \Zb^\pert_{e,i}:=\Zb^\base_{e,i}+\Wb\ab_e,
    \label{eq:perturbation_model}
    \qquad 
    \epsilonb_{e,i} \sim \QQ_\epsilonb,
    \qquad
    \Xb_{e,i}:=\dec\big(\Zb^\pert_{e,i}, \epsilonb_{e,i}\big),
\end{align}
where $(\Zb^\base_{e,i})_{e\in[M]_0, i\in [N_e]}$ are %
i.i.d.\ according to $\PP_\Zb$, and $(\epsilonb_{e,i})_{e\in[M]_0, i\in [N_e]}$ are i.i.d.\ according to~$\QQ_\epsilonb$ and jointly independent of $(\Zb^\base_{e,i})_{e\in[M]_0, i\in [N_e]}$.

\looseness-1 
The \textit{basal state} $\Zb^\base$ %
describes the unperturbed state of latent variables, which can, in principle, be affected by perturbations, and is distributed according to a base distribution~$\PP_\Zb$. 
The \textit{perturbation matrix} $\Wb\in\RR^{\dimZ\times K}$ 
captures the effect of the $K$ elementary perturbations encoded in $\ab$ on the latents and turns $\Zb^\base$ into  \textit{perturbed latents} $\Zb^\pert$.
Since the same perturbation~$\ab_e$ is applied for all $i\in[N_e]$, all within-dataset variability in $\Zb^\pert$ is due to $\PP_\Zb$.

The \text{noise variables} $\epsilonb\in\RR^{\dimeps}$
capture all other variation in the observed data that is unaffected by perturbations.
W.l.o.g., we assume that it is distributed according to a fixed, uninformative distribution $\QQ_\epsilonb$ such as a standard isotropic Gaussian.
The noise serves as an additional input to the \text{(stochastic) mixing function} or generator
$\dec:\RR^\dimZ\times \RR^\dimeps\to\RR^\dimX$%
, which produces observations for the perturbed latent.
This implicit
generative model can capture any conditional distribution $\PP_{\Xb|\Zb=\zb}$~\citep[e.g.,][Prop.~7.1]{peters2017elements} and
is more flexible than relying on parametric assumptions, such as a Gaussian~\citep{lotfollahi2023predicting} or negative Binomial~\citep{bereket2023modeling,lopez2018deep,lopez2023learning} 
likelihood parametrised by a deterministic decoder.

For a given $e$ and $\ab_e$, the generative process 
in~\eqref{eq:perturbation_model} %
induces a distribution %
 over observations~$\xb$, which we denote by $\PP_{e}$ or $\PP_{\Xb|\ab_e}$, %
defined as the push-forward of $\PP_\Zb$ and $\QQ_\epsilonb$ through~\eqref{eq:perturbation_model} %
such that %
\begin{equation}
\label{eq:data_distributions}
    \forall e\in[M]_0: \qquad 
    (\Xb_{e,i})_{i \in [N_e]}\overset{\text{i.i.d.}}{\sim}\PP_{e}\,.
\end{equation}
\section{Theoretical results}
\label{sec:theory}
We now present our theoretical analysis for the model class from~\cref{sec:model}. %
Here, we assume that the decoder is deterministic, so that we can write ${\dec:\RR^\dimZ\to\RR^\dimX}$, see~\cref{remark:deterministic_vs_noisy_mixing} for further discussion.

\subsection{Identifiability}
\looseness-1
We first study
identifiability, i.e., 
under which assumptions and up to which ambiguities certain parts of the postulated generative process can be provably recovered 
given access to the full distributions.
As established by the following results, our model class is identifiable up to orthogonal transformation, provided that the training perturbations are sufficiently diverse, the dimension of the latent space is known, and 
some additional technical assumptions hold.
All proofs are provided in~\cref{app:proofs}.
\begin{restatable}[Identifiability up to orthogonal transformation%
]{theorem}{idgaussian}
\label{thm:affine_identifiability_gaussian}
For $M\in \ZZ_{\geq 0}$, let $\ab_0, ..., \ab_M\in\RR^K$ be ${M+1}$ perturbation labels.
Let $\dec,\dect:\RR^{\dimZ}\to\RR^\dimX$, 
$\Wb,\Wbt\in\RR^{\dimZ\times K}$, and $\PP,\PPt$ be distributions on~$\RR^{\dimZ}$ such that the models $(\dec,\Wb,\PP)$ and $(\dect,\Wbt,\PPt)$ induce the same observed distributions, i.e.,
\begin{equation}
\label{eq:same_observed_distributions}
    \forall e\in[M]_0:%
    \qquad 
    \dec\left(\Zb+\Wb\ab_e\right)
    \overset{d}{=}
    \dect\Big(\Zbt+\Wbt\ab_e\Big),
\end{equation}
with $\Zb\sim\PP$ and $\Zbt\sim\PPt$ independent. Assume further that: 
\begin{enumerate}[label=(\roman*),leftmargin=1.6em,itemsep=-0.25em,topsep=-0.25em]
    \item \textbf{[invertibility]} $\dec$ and $\dect$ are $C^2$-diffeomorphisms onto their respective images;
    \item \looseness-1 \textbf{[Gaussianity]} $\Zb$ and $\Zbt$ are standard isotropic Gaussians with means chosen s.t.\ the latent distributions in domain $e=0$ are centered at the origin:
    $\PP=\Ncal\left(-\Wb\ab_0,\Ib\right)$ and $\PPt=\Ncal(-\Wbt\ab_0,\Ib)$;
     \item \textbf{[sufficient diversity]} the matrix $\Wbt\Ab \in\RR^{\dimZ \times M}$, where %
     $\Ab\in\RR^{K\times M}$ is the matrix with columns $(\ab_1-\ab_0)$, ..., $(\ab_M-\ab_0)$, has full row rank, i.e., $\rank(\Wbt\Ab)=\dimZ$, or the same holds for $\Wb\Ab$.
\end{enumerate}
\looseness-1 Then the representation and the effects of observed perturbations relative to~$\ab_0$ (captured by $\Wb\Ab$) are identifiable up to orthogonal transformation, i.e.,
there is an orthogonal matrix $\Ob\in O(\dimZ)$ such that
\begin{align}
\label{eq:identifiability_condition_f}
    \forall \zb\in\RR^\dimZ:\qquad \dect^{-1} \circ\dec (\zb)&=\Ob\zb, 
    \qquad \text{and}\qquad %
    \Wbt\Ab =\Ob \Wb\Ab.
\end{align}
\end{restatable}
\begin{restatable}
{corollary}{corollaryidgaussian}
\label{cor:id_gaussian}
If, in addition to the assumptions of~\cref{thm:affine_identifiability_gaussian},     $\Ab\in\RR^{K\times M}$ has full row rank (i.e., $\rank(\Ab)=K\leq M$),
then the perturbation matrix $\Wb$ is identifiable up to orthogonal transformation, i.e.,
    $\Wbt=\Ob\Wb$,
where $\Ob\in O(\dimZ)$ is an orthogonal matrix.
\end{restatable}
\begin{remark}[Choice of base distribution]
\label{remark:choice_of_base_distribution}
\looseness-1
    The choice of means and covariances in assumption~\textit{(ii)} of~\cref{thm:affine_identifiability_gaussian} serves to eliminate inherent ambiguities of the model class due to overparametrisation. As shown by~\cref{lemma:mean_covariance_ambiguity_no_intercept}, this comes w.l.o.g., in the sense that any model with $\PP=\Ncal(\mub,\Sigmab)$ can be reparametrised to take the form in assumption~\textit{(ii)}.
\end{remark}
\begin{remark}[Sufficient diversity]
\label{remark:sufficient_diversity}
\looseness-1 
The matrix product $\Wb\Ab\in\RR^{\dimZ\times M}$ captures the relative effects of the observed perturbations since $(\Wb\Ab)_{je}$ corresponds to the shift in the $j$\textsuperscript{th} latent $Z_j$ resulting from~$\ab_e$, relative to a reference condition $\ab_0$.
Moreover, assumption~\textit{(iii)} of~\cref{thm:affine_identifiability_gaussian} implies 
\begin{equation}
\label{eq:lower_bound_on_ranks}
\min\big\{\rank\big(\Wbt\big), \rank(\Ab)\big\}\geq \rank(\Wbt\Ab)=\dimZ.
\end{equation}
\looseness-1 Hence, sufficient diversity requires at least $\dimZ$ elementary perturbations whose associated shift vectors $\wb_k\in\RR^{d_Z}$ are linearly independent, and we must observe at least $\dimZ$ perturbation conditions~$\ab_e$ other than $\ab_0$ such that the relative perturbation vectors $(\ab_e-\ab_0)\in\RR^K$ are linearly independent.
\end{remark}
\begin{remark}[Choice of reference]
\label{remark:choice_of_reference}
    The choice of reference environment is arbitrary. 
    Here, we choose $e=0$ as reference without loss of generality.
    Intuitively, if a perturbation is always present (e.g., $a_{e,1}=1$ for all $e$), then its effects cannot be discerned from the basal state. 
    Therefore, only the effects of the relative perturbations $(\ab_e-\ab_0)$ can be recovered. 
    In practice, we often have access to an unperturbed, purely observational control condition with~$\ab_0=\bm 0$.
\end{remark}
\begin{remark}[Deterministic vs noisy mixing.]
\label{remark:deterministic_vs_noisy_mixing}
The mixing function~$\dec$ in~\cref{thm:affine_identifiability_gaussian} is assumed deterministic, i.e., does not take noise~$\epsilonb$ as input, cf.~\eqref{eq:perturbation_model}.
In principle, 
this does not pose a restriction since 
noise can be appended to $\zb^\pert$ as additional dimensions that are not influenced by perturbations. However, this increases~$\dimZ$ and makes it harder to satisfy sufficient diversity (see~\cref{remark:sufficient_diversity}).
Alternatively, the additive noise setting, $\Xb=\dec(\Zb)+\epsilonb$, can be reduced to the noiseless case~\citep{khemakhem2020variational}.
\end{remark}
\subsection{From identifiability to extrapolation}
\looseness-1 
Since we aim to make distributional predictions for new perturbations~$\ab_\te$, identifiability is only
of intermediary interest.
The following result highlights the usefulness of identifiability up to orthogonal transformation established in~\cref{thm:affine_identifiability_gaussian} for extrapolation.
It allows us to uniquely predict the observable effects of certain unseen perturbations---those which can be expressed as linear combinations of the observed perturbations, as illustrated in~\cref{fig:span}.

\begin{minipage}{0.6\textwidth}
\begin{restatable}[Extrapolation to span of relative perturbations]{theorem}{extrapolation}
\label{thm:extrapolation}
Under the same setting and assumptions as in~\cref{thm:affine_identifiability_gaussian}, let ${\ab_\te\in\RR^K}$ be an unseen perturbation~s.t.
\begin{equation}
    \label{eq:span}
    \left(\ab_\te-\ab_0\right)
    \in 
    \mathrm{span}
    \big(
    \left\{
    \ab_e-\ab_0
    \right\}_{e\in[M]}
    \big).
\end{equation}
Then the effect of $\ab_\te$ is uniquely identifiable in that
\begin{equation}
\label{eq:equal_test_prediction}
\Xb_\te=\dec(\Zb+\Wb\ab_\te)\overset{d}{=}\dect\big(\Zbt+\Wbt\ab_\te\big)
\end{equation}
\end{restatable}
\end{minipage}\hfill
\begin{minipage}{0.35\textwidth}
\centering
\includegraphics[width=\columnwidth, trim={0.4em 0.5em 0.5em 2.85em}, clip]{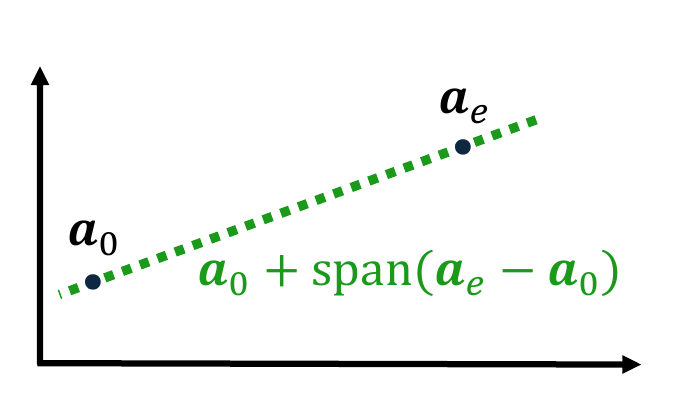}
\captionof{figure}{\small Illustration of %
Eq.~\eqref{eq:span}.}
\label{fig:span}
\end{minipage}
\vspace{0.5em}
\begin{remark}[Additive vs general shifts]
    \looseness-1 
    For identifiability, 
    it is not necessary that the mean shifts take the additive form $\Wb\ab_e$. %
    If we replace $\Wb\ab_e$ and $\Wbt\ab_e$ in~\eqref{eq:same_observed_distributions} with arbitrary shift vectors ${\cbb_e,\cbt_e\in\RR^{\dimZ}}$, the same result as~\cref{thm:affine_identifiability_gaussian} can be shown to hold with $\Wb\Ab$ and $\Wbt\Ab$ replaced by $\Cb$ and $\Cbt$, defined as the matrices with columns $(\cbb_e-\cbb_0)$ and $(\cbt_e-\cbt_0)$, respectively. That is, the relative shift vectors are identifiable (up to orthogonal transformation), regardless of whether they are linear in~$\ab$.
    This is relevant for including nonlinear interactions or learnable element-wise nonlinear dose-response functions $h_j(a_j)$~\citep[][]{lotfollahi2023predicting}.
    However, linearity in $\ab$ is leveraged to prove extrapolation, %
    where~\eqref{eq:span} establishes a link between $\ab_\te$ and the training perturbations.
    Since only $\ab_\te$ is observed at test time, 
    \cref{thm:extrapolation}
    cannot easily be extended to %
    arbitrary shifts $\cbb_\te$, as this would require establishing a link between $\cbb_\te$ and the training shifts, all of which are unobserved.%
\end{remark}

\begin{example}[Identifiability without extrapolation]
\label{ex:id_without_extrapolation}
\looseness-1 
Let $\dimZ\!=\!2$ and consider the perturbation labels from~\cref{ex:gene_perturbations} with $\Wb=(\wb_1,\wb_2,\wb_3)$, where $\wb_1\!=\!(1,0)^\top$, $\wb_2\!=\!(0,1)^\top$, and ${\wb_3\!=\!(1,1)^\top}$. If only $\{\ab_0,\ab_1,\ab_2\}$ are available for training, then $\Wb\Ab=\Ib$ and  sufficient diversity holds, but we still cannot extrapolate to $\ab_3=(0,0,1)^\top\not\in\mathrm{span}(\{\ab_2-\ab_0,\ab_1-\ab_0\})=\{(y,z,0)^\top:y,z\in\RR\}$. 
\end{example}
\subsection{Incorporating prior knowledge and extrapolation to completely unseen perturbations}
\looseness-1 
Encoding perturbations as sparse binary vectors as in~\cref{ex:gene_perturbations} makes it impossible to extrapolate to completely unseen perturbations via~\cref{thm:extrapolation} since~\eqref{eq:span}, by definition, does not hold in this case, as illustrated in~\cref{ex:id_without_extrapolation}.
To make use of prior knowledge about similarities among perturbations, suppose that, in addition to the data from~\eqref{eq:data}, we have access to an embedding $\phib_k\in\RR^{\dimPhi}$ for each elementary perturbation $k\in [K]$.
Such perturbation embeddings can come, e.g., from existing literature, alternative data sources, or pre-trained foundation models.
In the context of genetic perturbations, for example, \citet{he2025morph} explore different types of gene embeddings (though without the latent shift assumption and with a learned nonlinear pertubation encoder applied to the  embeddings).

\looseness-1 If the embeddings can be combined additively as $\Phib\ab_e$ with $\Phib=(\phib_1, ..., \phib_{K})$ and shifts are linear in the joint embeddings (i.e., $\Zb^\pert=\Zb^\base+\Wb\Phib\ab_e$ for $\Wb\in\RR^{\dimZ\times \dimPhi}$),
then \cref{thm:extrapolation} directly applies with~$\ab_e$ replaced by~$\Phib\ab_e$.\footnote{Importantly, sufficient diversity~(\cref{thm:affine_identifiability_gaussian}~\textit{(iii)}) can only hold if $\dimPhi\geq\dimZ$, cf.~\cref{remark:sufficient_diversity}.}
Crucially, this relaxes~\eqref{eq:span} to the following weaker condition 
\begin{equation*}
\Phib\left(\ab_\te-\ab_0\right)
    \in 
    \mathrm{span}
    \big(
    \left\{
    \Phib\left(\ab_e-\ab_0\right)
    \right\}_{e\in[M]}
    \big)    
\end{equation*}
which is implied by but does not imply~\eqref{eq:span}.
Intuitively, prior knowledge can facilitate extrapolation to completely unseen perturbations by relating them to seen ones in the embedding space.
\begin{example}[Extrapolation with prior-knowledge-based perturbation embeddings]
For the setting from~\cref{ex:id_without_extrapolation}, let ${\phib_1=(1,1,1,1)^\top}$, ${\phib_2=(1,-1,-1,1)^\top}$, and ${\phib_3=(1.2,0.8,0.8,1.2)^\top}$. Then extrapolation to $\ab_3$ becomes possible since ${\Phib\ab_3=\phib_3\in\mathrm{span}(\{\Phib\ab_1,\Phib\ab_2\})=\mathrm{span}(\{\phib_1,\phib_2\})}$. 
\end{example}

\section{Estimation method: Perturbation distribution autoencoder (PDAE)}
\label{sec:method}
\looseness-1
To leverage the extrapolation guarantees of~\cref{thm:extrapolation}, we seek to estimate
the identifiable 
 parts of the generative process 
  by matching the distributions of the observed
 data in~\eqref{eq:data} across all perturbation conditions, as prescribed by condition~\eqref{eq:same_observed_distributions} in~\cref{thm:affine_identifiability_gaussian}.
To this end,
we take inspiration from prior distributional learning approaches%
~\citep{bouchacourt2016disco,shen2024distributional} and
adapt them for multi-domain perturbation modeling.

\looseness-1 
Our method, the \textit{perturbation distribution autoencoder}~(PDAE), %
comprises an encoder%
, a perturbation matrix%
, and a (stochastic) decoder, %
which are trained to maximise the similarity between pairs of true and simulated perturbation distributions, see~\cref{fig:architecture} for an overview.
 Intuitively, instead of specifying a (Gaussian, as per~\cref{thm:affine_identifiability_gaussian}) prior on the  basal state from which to sample latents as in~\eqref{eq:perturbation_model}, we use encoded observations from a source domain as perturbed latents and transform them to perturbed latents from a target domain using the learnt perturbation matrix and the assumed additivity in latent space.

\begin{figure}[t]
    \centering
    \includegraphics[width=\textwidth]{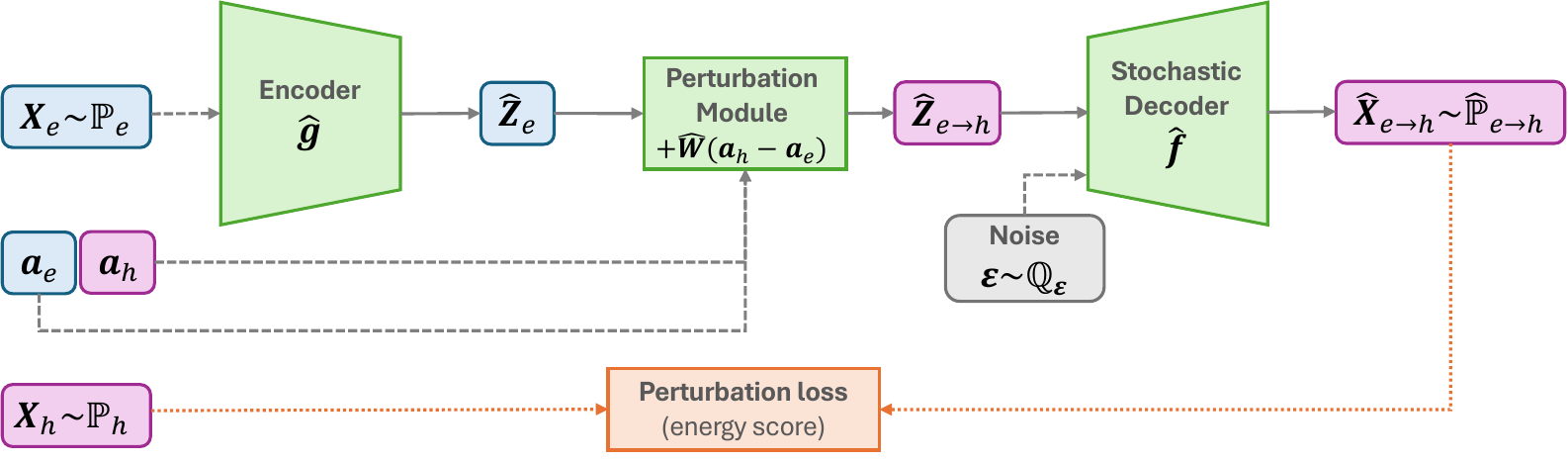}
    \caption{\small
    \looseness-1 
    \textbf{Perturbation distribution autoencoder~(PDAE).}
    PDAE simulates the distribution of a target perturbation condition $h$ \textit{(purple)} by encoding, perturbing, and (stochastically) decoding data from a source condition~$e$ \textit{(blue)}. 
    Dashed arrows indicate model inputs and green boxes model components with trainable parameters. The learning objective \textit{(orange, dotted arrows)} amounts to
    a perturbation loss \textit{(bottom)}, measuring, for all pairs of training domains $(e,h)$, the dissimilarity between the %
     true distribution of~$\Xb_h$
    and its simulated version %
    based on domain $e$.
    At test time, the target perturbation label $\ab_h$ can be replaced with an unseen~$\ab_\te$ to make predictions.}
    \label{fig:architecture}
\end{figure}

\textbf{Encoder.} 
\looseness-1 
The %
encoder
$\ench:\RR^{\dimX}\to\RR^{\dimZ}$ maps observations $\xb$ to the space of perturbation-relevant latents $\zb$. 
Ideally, it should invert the stochastic mixing function $f$ in~\eqref{eq:perturbation_model} in the sense of recovering the perturbed latent state~$\zb^\pert_{e,i}$%
from  observation~$\xb_{e,i}$%
. %
We therefore denote the encoder outputs by
    $\zbh^\pert_{e,i} := \ench(\xb_{e,i})$
and refer to them as \textit{estimated perturbed latents}. 

\textbf{Perturbation module.}
If the encoder recovers the perturbed latents up to orthogonal transformation, the additivity of perturbation effects assumed in~\eqref{eq:perturbation_model} allows us to map between the latent distributions underlying different perturbation conditions.
Specifically, we use a perturbation model parametrised by a perturbation matrix $\Wbh \in \RR^{\dimZ\times K}$ to create
\textit{synthetic perturbed latents} from domain $h$ as:
\begin{align}
    \label{eq:perturbed_latents}
    \zbh^\pert_{e\to h,i}:=
    \zbh^\pert_{e,i} + \Wbh(\ab_h-\ab_e).
\end{align}
\looseness-1 
which amounts to undoing the effect of %
$\ab_e$ (mapping back to the basal state) and then simulating%
~$\ab_h$.

\textbf{Decoder.}
The decoder $\dech:\RR^{\dimZ}\times \RR^\dimeps\to\RR^{\dimX}$ maps estimated perturbed latents $\zbh^\pert$ and noise $\epsilonb\sim\QQ_\epsilonb$ to observations. 
When evaluated on synthetic perturbed latents from~\eqref{eq:perturbed_latents}, we refer to the corresponding (random) outputs
\begin{equation}
\label{eq:synthetic_observations}
\Xbh_{e\to h,i}=\dech\big(\zbh^\pert_{e\to h,i}, \,\epsilonb%
\big) \qquad \text{where} \qquad \epsilonb%
\sim\QQ_\epsilonb
\end{equation}
\looseness-1
as \textit{synthetic observations} from domain $h$ based on domain~$e$.

\textbf{Simulated perturbation distributions.} 
Given a distribution~$\PP_e$ and perturbation label~$\ab_e$, our model facilitates sampling synthetic observations for another perturbation condition with label~$\ab_h$ via~\eqref{eq:perturbed_latents} and~\eqref{eq:synthetic_observations}.
We denote the resulting distribution by $\PPh_{e\to h}$, formally defined as the distribution of
\begin{equation}
\label{eq:induced_synthetic_distribution}
\dech\big(\ench(\Xb_e)+\Wbh(\ab_h-\ab_e),\epsilonb\big),
 \qquad \text{where} \qquad  \Xb_e\sim\PP_e \qquad \text{and} \qquad 
 \epsilonb\sim\QQ_\epsilonb.
\end{equation}
Intuitively, when simulating 
$\PP_h$ based on $\PP_e$ via~\eqref{eq:induced_synthetic_distribution}, the random variable ${\ench(\Xb_e)-\Wbh\ab_e}$ in the first argument of~$\dech$ plays the role of a random latent basal state from~\eqref{eq:perturbation_model}.

\textbf{Population-level learning objective.}
\looseness-1 
We propose minimizing a sum of pairwise distribution losses, %
\begin{equation}
    \label{eq:combined_loss}
    \argmin_{\ench,\dech,\Wbh} 
    \sum_{e,h\in[M]_0}
    d\left(\PPh_{e\to h}, \PP_h\right)
    \quad \text{with} \quad 
    d\big(\PPh_{e\to h}, \PP_h
    \big)=
    -\EE_{\Xb_h\sim\PP_h}
    \big[\ES_\beta
    \big(\PPh_{e\to h}, \Xb_h
    \big)
    \big]
\end{equation}
\looseness-1 
where $\PPh_{e\to h}$ depends on $(\PP_e,\ab_e,\ab_h)$ and the model parameters via~\eqref{eq:induced_synthetic_distribution}; and $d$ measures distributional dissimilarity %
through the negative expected energy score,
with $\ES_\beta$ \citep{gneiting2007strictly} defined as
\begin{equation*}
\textstyle
    \ES_\beta(\PP,\xb)=\frac{1}{2}\EE_{\Xb,\Xb'\,\overset{\text{i.i.d.}}{\sim}\,\PP} \norm{\Xb-\Xb'}^\beta-\EE_{\Xb\sim\PP} \norm{\Xb-\xb}^\beta.
\end{equation*}%
\looseness-1 For $\beta\in(0,2)$, $\ES_\beta$ is a strictly proper scoring rule, meaning that the expected energy score $\EE_\Xb[\ES_\beta(\PP,\Xb)]$ is maximized if and only if $\Xb\sim\PP$, see~\cref{sec:forecasting} for details. %
Combined with its computational simplicity, this property makes the negative expected energy score a popular loss function for distributional regression~\citep{bouchacourt2016disco,shen2024engression,shen2024distributional,de2025distributional,shen2025reverse}.
It also directly implies the following corollary.%
\begin{corollary}
The objective in~\eqref{eq:combined_loss} is minimized if and only if $\PP_h=\PPh_{e\to h}$ for all $e,h\in[M]_0$.%
\end{corollary}%
\textbf{Training.}
Since, in practice, we do not have access to the true distributions, %
we approximate  expectations w.r.t.\ $\PP_e$ and $\PP_h$ in~\eqref{eq:combined_loss} %
with Monte Carlo estimates based on mini-batches of~\eqref{eq:data}, and use a single %
 draw of $\epsilonb$ to approximate expectations w.r.t.~$\QQ_\epsilonb$.
We then
optimise~\eqref{eq:combined_loss} w.r.t.\ $\Wbh$ and the parameters of $\ench$ and~$\dech$ using stochastic gradient descent~\citep{robbins1951stochastic,kingma2014adam}, see~\cref{alg:training} in~\cref{app:algorithm}.

\textbf{Prediction.}
To sample from the predicted distribution for a new perturbation label $\ab_\te$, we use our model to compute the synthetic perturbed test latents $\zbh^\pert_{e \to\te,i}$ for all $e\in[M]_0$ and all $i\in[N_e]$ via~\eqref{eq:perturbed_latents}, and then sample %
synthetic test observations $\Xbh_{e\to\te, i}$ according to~\eqref{eq:synthetic_observations}.
Our estimate of $\PP_\te$ is a convex combination of the domain-specific marginal perturbation distributions $\PPh_{e\to\te}$ induced by our model for $\ab_\te$, i.e., the empirical version of
    $\PPh_\te = 
    \sum_{e\in[M]_0}
    \omega_e
    \PPh_{e\to\te}$,
\looseness-1 
with weights $\omega_e\geq 0$ such that $\sum_e\omega_e=1$ to account for prior knowledge or differing sample sizes, e.g., uniform weighting ($\omega_e\!=\!\frac{1}{M+1}$) or using only control observations ($w_e=0$ for $e\neq0$).

\textbf{Full algorithm and extensions.}
Due to space constraints, we present only the core PDAE architecture and learning objective here;  optional additional loss components %
(a reconstruction objective, latent prior, and sparsity regularization) and a complete algorithm
are detailed in~\cref{app:algorithm}.

\section{Related work}
\textbf{Identifiable representation learning for extrapolation.}
\citet{saengkyongam2024identifying} consider a similar identifiability problem  to the one we study in~\cref{sec:theory} and leverage affine identifiability to derive out-of-support extrapolation guarantees. However, they focus on continuous perturbations~$\ab$ and require the training support to contain a non-empty open subset of $\mathbb{R}^K$. As a result, their model, proof strategy, and estimation method (based on maximum moment restrictions) differ from ours. 

\textbf{Causal approaches.}
\looseness-1 
\citet{schneidergenerative} model perturbations as interventions on the measurements.
\citet{zhang2024identifiability} take a causal representation learning~(CRL)~\citep{scholkopf2021toward} approach and model perturbations as shift interventions in a latent causal model. %
They use a discrepancy-based VAE, which, similar to PDAE, is trained primarily to maximize the similarity between real and simulated perturbation distributions~\citep{liu2025learning}. 
In~\cref{app:relation_to_SCMs}, we show that our model can also be interpreted as interventional CRL~\citep{von2023nonparametric,varici2025score} with a linear latent causal model~\citep{squires2023linear,buchholz2024learning} and generalized shift interventions~(\cref{prop:sems_special_case}).

\textbf{Synthetic interventions.}
The type of extrapolation to linear combinations in~\cref{thm:extrapolation} 
resembles similar conditions for synthetic interventions~(SI)~\citep{agarwal2020synthetic,agarwal2023synthetic,squires2022causal,jiralerspong2024general}.
In contrast to our nonlinear setting, the SI literature relies on linear factor models and methods for matrix or tensor completion.
Moreover, SI methods generally aim to impute unit-level potential outcomes 
whereas we target the distribution-level.
A notable exception is the framing of~\citet{squires2022causal}, which is more in line with our population-level viewpoint, but only identifies the mean of the distribution. 

\textbf{Testing for and modeling interactions.}
\looseness-1 
\citet{xu2024automated} propose tests for the hypthesis of no (nonlinear) interactions in latent space, based on the assumption that perturbations affect disjoint latents and given suitable data from both individual and pairs of perturbations.
\citet{adduri2025predicting} and \citet{he2025morph} propose modeling interactions via attention in transformer-based architectures~\citep{vaswani2017attention}. 

We discuss further related work on modeling biological perturbations in~\cref{sec:CPA}.

\begin{figure}
    \centering
    \begin{subfigure}[b]{0.37\columnwidth}
        \includegraphics[width=\columnwidth,trim={0.4em 0 0 0}, clip]{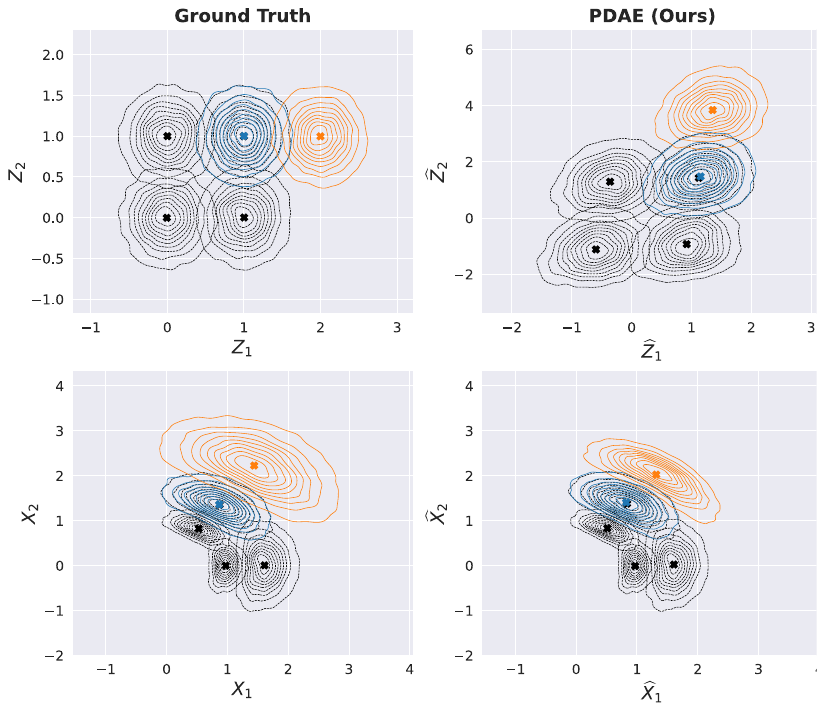}
        \caption{Qualitative results}
        \label{subfig:qualitative_results_main}
    \end{subfigure}\hfill
\begin{subfigure}[b]{0.62\columnwidth}
    \includegraphics[width=\columnwidth]{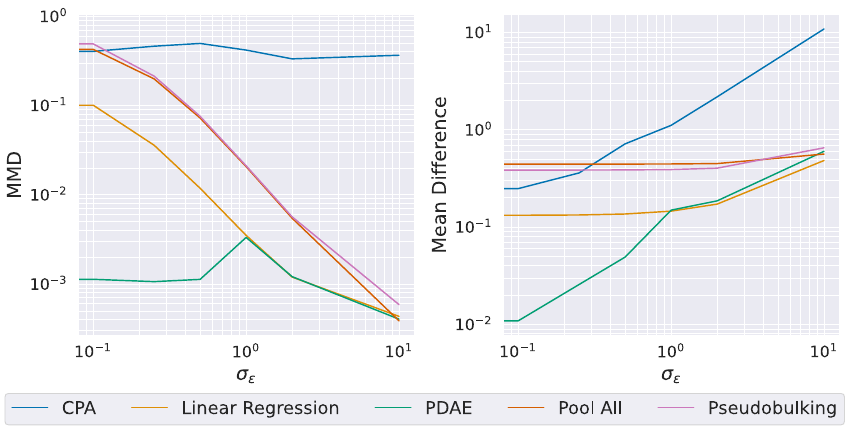}
    \caption{Varying signal-to-noise ratio}
    \label{subfig:noise_comparison}
\end{subfigure}
    \caption{\small \looseness-1
    \textbf{Results on synthetic data.}
    \textbf{(a)}~%
    Kernel density estimates of the training data \textit{(black)}, an in-distribution (ID) test case \textit{(blue)} 
     which was not seen but happens to induce the same perturbed latent distribution as one of the training perturbations, 
     and an out-of-distribution (OOD) test case \textit{(orange)}.
    PDAE recovers an affine transformation---since condition \textit{(ii)} of~\cref{thm:affine_identifiability_gaussian} is not enforced, cf.~\cref{app:reparametrizability}---of the true latents \textit{(top row)} leading to
    accurate distributional predictions for the training and ID test domains \textit{(bottom row)}. 
    In the OOD case, the test distribution is identified in latent space  but lies mostly outside the training support \textit{(orange, top right)}. 
    As a result, the decoder output does not fully match the true OOD distribution \textit{(orange, bottom row)}.
    \textbf{(b)}~PDAE performs best for small noise variance $\sigma_\epsilon^2$ and comparable to its linear version for intermediate noise. For large noise, distributions become indistinguishable based on MMD \textit{(left)} and mean prediction \textit{(right)} deteriorates to baseline level.
    }
    \label{fig:qualitative_results_main}
\end{figure}

\section{Experimental validation}
We investigate the empirical finite-sample behaviour of PDAE in a controlled setting with synthetic data~(\cref{sec:experiments}) and then showcase it through a case study on combinatorial gene perturbation data~(\cref{sec:real_world_case_study}).

\subsection{Simulation study}
\label{sec:experiments}

\textbf{Data.}
\looseness-1 
For ease of visualization, we use $\dimZ\!=2$-dim.\ latents.
We consider $K\!=\!3$ elementary perturbations and generate training and test domains using the perturbation labels from~\cref{ex:gene_perturbations} with shift vectors $\wb_1\!=\!(1,0)^\top$, $\wb_2\!=\!(0,1)^\top$, ${\wb_3\!=\!(1,1)^\top}$ and base distribution $\PP_\Zb=\Ncal(\bm 0,0.25^2\Ib)$.
In this case, the sufficient diversity condition~\textit{(iii)} of~\cref{thm:affine_identifiability_gaussian} is satisfied.
We generate observations as ${\xb\!=\![\dec(\zb); \epsilonb]}$ where $\dec(\zb)\!=\!\mathrm{e}^{z_1}(\cos z_2,\sin z_2)$ and $\epsilonb\sim\Ncal(\bm 0, \sigma^2_\epsilon\Ib)$ is concatenated noise.

\textbf{Setup.}
\looseness-1 
We compare PDAE with CPA~\citep{lotfollahi2023predicting} and baselines therein which pool either all observations or observations from perturbations involved in the combination (``Pseudobulk'').
We also propose a new baseline, which learns additive shifts directly in observation space by linearly regressing the domain-specific means $\mub^{\xb}_e\in\RR^\dimX$ on $\ab_e$ and using the resulting model to predict $\mub_\te^{\xb}$ from~$\ab_\te$. This amounts to a special case of PDAE with encoder and decoder equal to the identity and $d_Z=d_X$.
We report distributional fit in terms of maximum mean discrepancy~(MMD)~\citep{gretton2012kernel} and the difference between predicted and true mean, $\norm{\mub^\xb-\widehat{\mub}^\xb}$.
For further experimental details, see~\cref{app:experimental_details}.

\textbf{Results.} \Cref{subfig:qualitative_results_main} shows the true vs estimated observed and latent distributions in the noiseless case and \cref{subfig:noise_comparison} shows mean and distributional fit for $\dimeps\!=\!8$ and $\dimX\!=\!10$ across different noise levels, see the caption for additional details.
Additional quantitative results for larger numbers of randomly sampled test cases are summarized in~\cref{tab:results} in~\cref{app:summary_synthetic_data}, where PDAE achieves near perfect mean and distributional fit, clearly outperforming its linear instantiation in second place.

\textbf{Takeaway.} When our assumptions hold, PDAE accurately predicts unseen perturbation distributions, but limited support overlap in latent space and low signal-to-noise ratio can be  challenges in practice.

\subsection{Real-world case study: combinatorial gene perturbations (PerturbSeq)}
\label{sec:real_world_case_study}
\textbf{Data.} We evaluate our method on two combinatorial PerturbSeq datasets spanning different cell types and experimental technologies. Data from~\citet{norman2019exploring} contains a total of $\sim$88.6k K562 leukemia cells across 236 distinct CRISPR activation perturbations (131 doubles).
Data from \citet{wessels2023efficient} contains a total of $\sim$22.8k monocyte immune cells across 185 CRISPR interference perturbations (158 doubles).
For both datasets, each cell contains measurements of $\dimX\approx$ 8.2k genes.

\textbf{Setup.} \looseness-1 We adopt the experimental setup of~\citet{miller2025deep}. Control cells and single perturbations are used for training; doubles are split across train (25\%), validation (25\%), and test (50\%). Each method is trained and tested across 2 folds so that each double perturbation is used exactly once for evaluation.

\textbf{Methods.} 
We compare PDAE with: its linear regression instantiation; using only control (``Control Mean'') or all cells (``Dataset Mean'')~\citep{vinas2025systema}; a fixed additive model (${\mub_{AB}:=\mub_A+\mub_B-\mub_\text{ctrl}}$)~\citep{gaudelet2024season,ahlmann2025deep}; an oracle, that uses half the test cells to predict the other half (``Technical Duplicate'')~\citep{miller2025deep}; and existing deep learning approaches developed for single-cell perturbation modeling---scLambda~\citep{wang2024modeling}, GEARS~\citep{roohani2024predicting}, and PRESAGE~\citep{littman2025gene}. Both GEARS and PRESAGE rely on prior biological knowledge.
Further details on hyperparameter choices and model selection are described in~\cref{app:additional_details_and_results_gene_data}.

\textbf{Metrics.} 
\looseness-1 To meaningfully evaluate predictions in 8k-dim.\ gene expression space, we focus on metrics that are weighted by differential expression level, as argued for in several recent studies~\citep{miller2025deep,mejia2025diversity,vollenweider2026signal}. Specifically, we report the weighted mean squared error (wMSE)~\citep{mejia2025diversity} and the weighted correlation between the true and predicted difference vectors to the dataset mean baseline (wPearson$\Delta$Pert)~\citep{vollenweider2026signal}.

\begin{figure}[tbp]
\vspace{-1.75em}
    \centering
    \includegraphics[height=0.345\linewidth, trim={0 0 0em 5em}, clip]{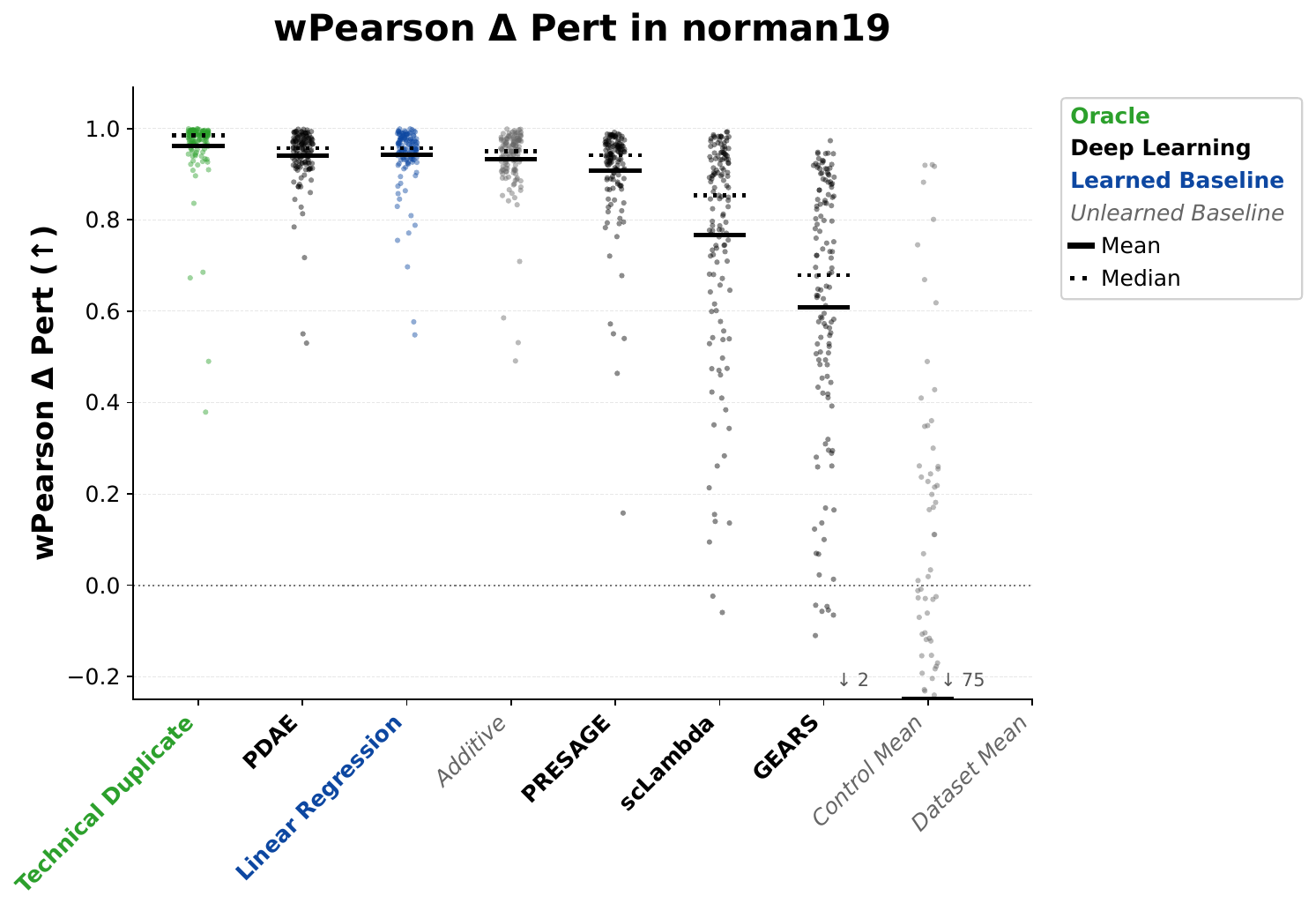}%
    \includegraphics[height=0.345\linewidth, trim={0 0 15em 5em}, clip]{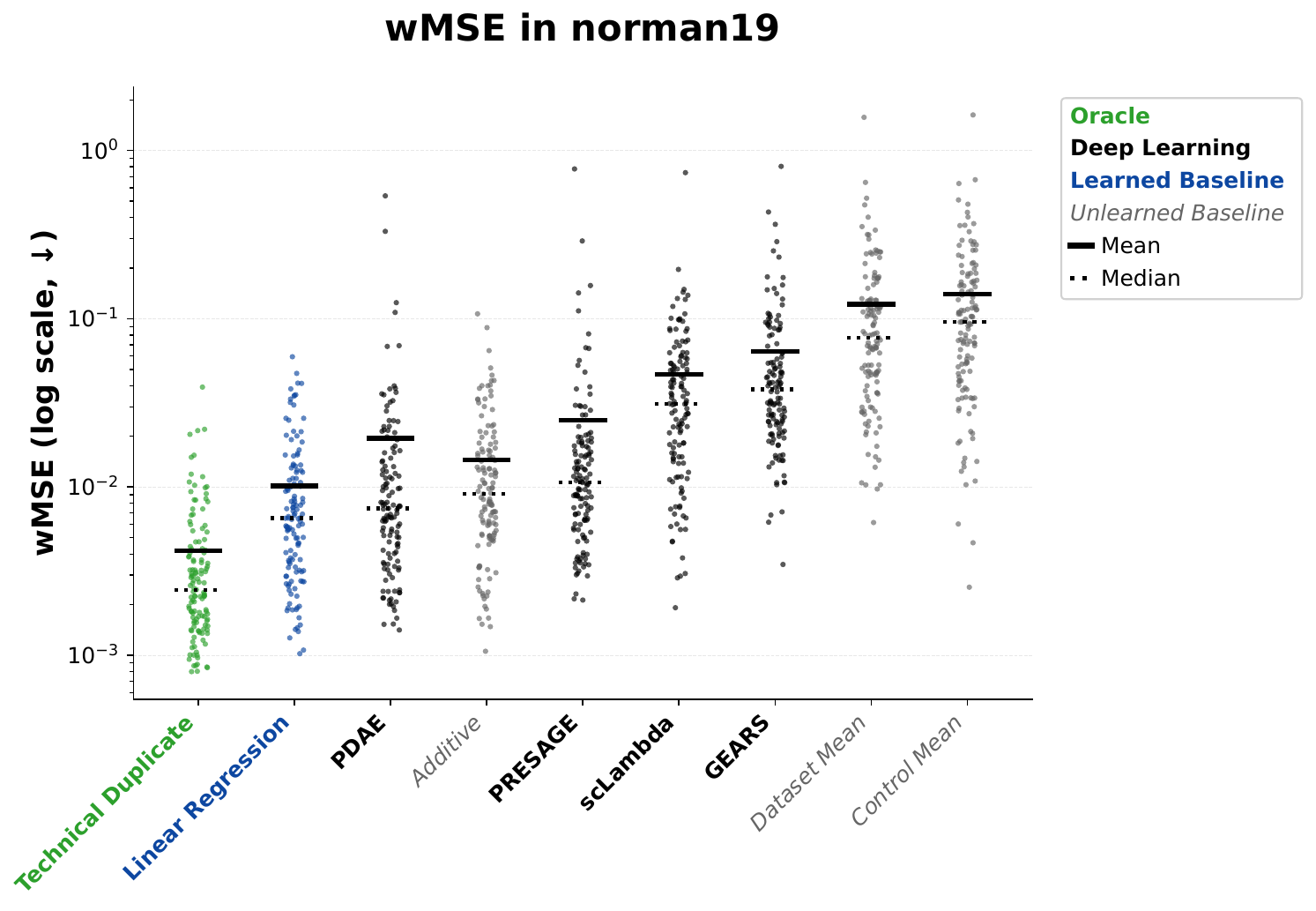}
    \vspace{-0.2em}
    \caption{\small \looseness-1 \textbf{Results for gene perturbation prediction.} Shown are wPearson$\Delta$Pert (left) and wMSE (right) on the CRISPRa combinatorial PerturbSeq data from~\citet{norman2019exploring}. Each dot in a strip-plot corresponds to a different test double perturbation. Due to the large spread, methods are sorted by median.
    PDAE performs better than existing deep learning based methods and baselines and comparable to its linear counterpart.} 
    \label{fig:results_norman}
    \vspace{-0.5em}
\end{figure}

\textbf{Results.} 
\looseness-1 The results on the data from~\citet{norman2019exploring} are shown in~\cref{fig:results_norman}, see the caption for details. 
The corresponding results for the data from~\citet{wessels2023efficient} in~\cref{fig:results_wessels} in~\cref{app:additional_details_and_results_gene_data} show similar trends, albeit with less pronounced differences between methods and a much larger gap to oracle performance.

\textbf{Takeaway.} \looseness-1 
Without extensive engineering, PDAE outperforms prior deep learning methods and baselines, underscoring the challenging nature of extrapolation and the benefit of principled approaches. Still, linear methods remain highly competitive on noisy high-dimensional gene expression data.

\section{Discussion, limitations, and future work}
\label{sec:discussion}
\looseness-1 
\textbf{Interpretation of results.}
\looseness-1 
Our somewhat sobering findings from~\cref{sec:real_world_case_study}---that deep learning models do not outperform additive baselines at combinatorial gene perturbation prediction---are consistent with recent benchmarks~\citep{ahlmann2025deep,vinas2025systema,vollenweider2026signal}.
\cref{subfig:noise_comparison} suggest the low signal-to-noise ratio in these datasets as a possible explanation.
Further gains may still be possible through  careful hyperparameter tuning
and explicitly modeling biological features such as the heavily zero-inflated nature of gene expression data and the limited effectiveness of CRISPR perturbations~\citep{hugi2025perturbation}.
Moreover, cell embeddings produced by representation learning approaches are often of independent interest, and our findings from~\cref{sec:experiments} show that nonlinear approaches can yield substantial improvements, especially in high signal-to-noise regimes.

\textbf{Linear vs additive baseline.}
\looseness-1
Linear regression follows the same additive shift model as PDAE but in observation space. Perhaps surprisingly, this baseline is not typically considered for perturbation prediction. 
The widely used additive baseline~\citep{ahlmann2025deep,gaudelet2024season} is a special case thereof, which is fitted only from single perturbations. 
Since double perturbations available during training provide additional information, future studies should report the more competitive linear regression baseline for fairer comparisons.

\textbf{Necessity of Gaussianity.}
\looseness-1 
Gaussianity of~$\PP_\Zb$ is sufficient for identifiability, but may not be necessary. %
We hypothesize that our results can be generalized to other exponential family distributions.
Indeed, PDAE does not enforce any particular latent distribution and still yields good performance.
In~\cref{app:algorithm}, we discuss how to impose a prior like the one from assumption~\textit{(ii)} of~\cref{thm:affine_identifiability_gaussian} by penalizing deviations of the encoded basal state distributions from the prior, e.g., using the energy score. 

\textbf{Decoder extrapolation.}
\looseness-1 
We have identified the need for the  decoder to generalize to unseen inputs as a separate challenge for distributional perturbation extrapolation. 
Although this issue is absent from our theory, where $\PP_\Zb$
 is Gaussian and has full support, it arises in practice when learning from finite data. 
As discussed in~\cref{fig:qualitative_results_main}, it
is orthogonal to learning the correct representation and perturbation model.
However, it may be possible to detect when perturbed test latents fall outside the training support
through support-overlap measures, outlier detection, or retraining on perturbed predictions. This can, in turn, enable
  uncertainty quantification for the predicted observation-space distributions.

\textbf{Concluding remarks.}
\looseness-1 
While we have focused on biological perturbations as a motivating running example, we emphasize that our theoretical and conceptual contributions are broader in scope.
Our empirical findings suggest theoretically grounded %
models based on the latent additive shift assumption like PDAE as a promising starting point, but further research and engineering is needed to make domain-specific perturbation models that extrapolate reliably a practical reality.
How to model interactions in a principled manner while maintaining extrapolation capabilities remains an open question.

\begin{ack}
    The authors thank Jiaqi Zhang and David M.\ Blei for insightful discussions. JvK is supported by The Branco Weiss Fellowship---Society in Science.
\end{ack}

\setlength{\bibsep}{4pt plus 2pt minus 2pt}
{
\bibliographystyle{abbrvnat}
\bibliography{ref}}

\clearpage
\appendix
\addcontentsline{parttoc}{section}{Appendices}
\part{Appendices}
\changelinkcolor{black}{}
\renewcommand\ptctitle{}
\parttoc
\changelinkcolor{BrickRed}{}

\section{Connection to causal representation learning}
\label{app:relation_to_SCMs}
In causal inference, experimental data resulting from perturbations is modeled via interventions in an underlying causal model.
In structural equation models~\citep[SEMs;][]{pearl2009causality, Bongers2020}, 
interventions modify a subset of assignments that determine each variable from its direct causes and unexplained noise. 
In general, our model for the effect of perturbations in latent space~(\cref{sec:model}) differs from how interventions are treated in SEMs. 
However, as summarized in the following result,
the class of shift interventions in linear SEMs constitutes a special case of our model class.%
\begin{restatable}[%
Modeling shift interventions in SEMs]{proposition}{sems}
    \label{prop:sems_special_case}
    \looseness-1
    Consider a linear 
    SEM,
        $\Zb:=\Bb^\top\Zb+\etab$, with noise~$\etab\sim\PP_{\etab}$
    and $\Bb\in\RR^{\dimZ\times\dimZ}$  a weighted adjacency matrix 
    with spectral radius (largest absolute eigenvalue) $\rho(\Bb)<1$.
    Then, the distribution induced by %
    shift interventions $\ab_e\in\RR^{\dimZ}$,
    \begin{equation}
        \label{eq:SEM_shift_interventions}
        \Zb:=\Bb^\top\Zb+\etab+\ab_e,
    \end{equation}
    is identical to that induced by~\eqref{eq:perturbation_model} with  $\Wb=(\Ib-\Bb^\top)^{-1}$ and basal state distribution $\PP_{\Zb}=(\Ib-\Bb^\top)^{-1}_{\#}\PP_{\etab}$.
\end{restatable}
\begin{remark}[Applicability to cyclic or confounded SEMs]
\cref{prop:sems_special_case} does not require the acyclicity or unconfoundedness (i.e., independence of the noise components) of the SEM.
The spectral radius condition (which holds for acyclic systems but is strictly weaker) ensures stability and invertibility of~$\Ib-\Bb^\top$ \citep{zheng2018dags,brown2025large}. 
\end{remark}
\looseness-1 
Combined with a decoder from~\eqref{eq:perturbation_model}, \cref{prop:sems_special_case}
highlights in which sense our model can be interpreted as interventional causal representation learning~\citep{scholkopf2021toward,von2023nonparametric,varici2025score} with a linear latent causal model~\citep{squires2023linear,buchholz2024learning} and generalized shift interventions~\citep{zhang2024identifiability}.
For $K=\dimZ$ elementary perturbations, the adjacency %
matrix corresponding to an inferred $\Wbh$ is given by~$\widehat{\Bb}=(\Ib-\Wbh^{-1})^\top$.
However, due to the rotation ambiguity, additional assumptions such as sparse (e.g., single node) shifts would be %
needed to recover the causal graph induced by the true~$\Bb$.

\clearpage
\section{Proofs}
\label{app:proofs}

\subsection{Proof of reparametrizability without loss of generality}
\label{app:reparametrizability}
We first state and prove the following lemma, which shows that the choices of means and covariances of the base distribution in assumption~\textit{(ii)} of~\cref{thm:affine_identifiability_gaussian} come without loss of generality (w.l.o.g.).
\begin{lemma}
\label{lemma:mean_covariance_ambiguity_no_intercept}
    Let $\fb$ and $\Wb$ be such that assumptions \textit{(i)}--\textit{(iii)} of~\cref{thm:affine_identifiability_gaussian} are satisfied, and let $\PP=\Ncal(\mub,\Sigmab)$ with positive definite covariance matrix $\Sigmab$.
    Then there exist $\dect$ and $\Wbt$ such that $\dect$, $\Wbt$, and $\PPt=\Ncal(-\Wbt\ab_0,\Ib)$ generates the same distributions, i.e., 
    \begin{equation}
    \label{eq:mean_covariance_ambiguity_no_intercept}
        \forall e\in[M]_0: \qquad \dec(\Zb+\Wb\ab_e)\overset{d}{=}\dect(\Zbt+\Wbt\ab_e), \qquad \text{where} \quad \Zb\sim\Ncal(\mub,\Sigmab) \quad \text{and} \quad \Zbt\sim\Ncal(-\Wbt\ab_0,\Ib),
    \end{equation}
    with $\dect$ and $\Wbt$ also satisfying assumptions \textit{(i)} and \textit{(iii)} of~\cref{thm:affine_identifiability_gaussian}.
\end{lemma}
\begin{proof}
    Define $\dect: \zb\mapsto \dec(\mub+\Wb\ab_0+\Sigmab^{\frac{1}{2}}\zb)$ and let $\Wbt:=\Sigmab^{-\frac{1}{2}}\Wb$. The equality in distribution in~\eqref{eq:mean_covariance_ambiguity_no_intercept} then follows directly by substitution from the properties of linear transformations of Gaussians.
    Since $\dec$ is a $C^2$ diffeomorphism by assumption, so is $\dect$, since it is the composition of $\dec$ with an affine function. Moreover, since $\Sigmab$ (and thus also any inverse square root $\Sigmab^{-\frac{1}{2}}$) is invertible, we have $\rank(\Wbt\Ab)=\rank(\Wb\Ab)$ and thus asssumption \textit{(iii)} also holds. 
\end{proof}

\subsection{Proof of~\texorpdfstring{\cref{thm:affine_identifiability_gaussian}}{}}
\label{app:proof_id_gaussian}

\idgaussian*

\begin{proof}
Let $p_e$ and $\pt_e$ denote the densities of 
\begin{equation}
\label{eq:ze}
    \Zb_e:=\Zb+\Wb\ab_e
\end{equation}
and
\begin{equation}
\label{eq:zetilde}
    \Zbt_e:=\Zbt+\Wbt\ab_e,
\end{equation}
respectively. 
Due to~\eqref{eq:same_observed_distributions}, $\dec$ and $\dect$ have the same image.
Thus, by the invertibility assumption \textit{(i)}, the function $\hb:=\dect^{-1}\circ \dec:\RR^{\dimZ}\to\RR^\dimZ$ is a well-defined $C^2$-diffeomorphism. 
The change of variable formula applied to 
\begin{equation}
  \Zbt_e\overset{d}{=}\hb(\Zb_e)  
\end{equation}
then yields for all $e$ and all~$\zb$:
\begin{equation}
\label{eq:change_of_var}
    p_e(\zb)=\pt_e\left(\hb(\zb)\right)\abs{\det \Jb_\hb(\zb)},
\end{equation}
where $\Jb_\hb(\zb)$ denotes the Jacobian of $\hb$.
By taking logarithms of~\eqref{eq:change_of_var} and contrasting domain $e$ with a reference domain with $e=0$, the determinant terms cancel and we obtain for all $e$ and all $\zb$:
\begin{equation}
\label{eq:domain_log_contrast}
    \log p_e(\zb) -\log p_0(\zb)
    =
    \log \pt_e\left(\hb(\zb)\right) -\log \pt_0\left(\hb(\zb)\right).
\end{equation}
Next, denote the densities of $\Zb$ and $\Zbt$ by $p$ and $\pt$, respectively.
From~\eqref{eq:ze} and~\eqref{eq:zetilde} it then follows that, for all $e$ and all $\zb$,  $p_e$ and $\pt_e$ can respectively be expressed in terms of $p$ and $\pt$ as follows:
\begin{align}
    p_e(\zb)&=p\left(\zb-\Wb\ab_e\right),
    \\
    \pt_e(\zb)&=\pt\left(\zb-\Wbt\ab_e\right).
\end{align}
By substituting these expressions in~\eqref{eq:domain_log_contrast}, %
we obtain for all $e$ and all $\zb$:
\begin{equation}
\label{eq:proof_starting_point}
    \log p(\zb-\Wb\ab_e) -\log p(\zb-\Wb\ab_0)
    =
    \log \pt\left(\hb(\zb)-\widetilde \Wb\ab_e\right) -\log \pt\left(\hb(\zb)-\widetilde \Wb\ab_0\right).
\end{equation}

By the Gaussianity assumption~\textit{(ii)}, the contrast of log-densities in~\eqref{eq:proof_starting_point} takes the following form for all $e$ and all $\zb$:
\begin{align}
    &(\ab_e-\ab_0)^\top \Wb^\top (\zb+\Wb\ab_0) & -\frac{1}{2}(\ab_e-\ab_0)^\top \Wb^\top \Wb (\ab_e+\ab_0)
    \\
    =\,&(\ab_e-\ab_0)^\top \Wbt^\top \left(\hb(\zb)+\Wbt\ab_0\right) & -\frac{1}{2}(\ab_e-\ab_0)^\top \Wbt^\top \Wbt (\ab_e+\ab_0)
\end{align}

We aim to show that
the two representations are related by an affine transformation, i.e., that all second-order derivatives
of $\hb$ are zero everywhere.
Taking gradients w.r.t.\ $\zb$ yields for all $e$ and all~$\zb$:
\begin{equation}
    (\ab_e-\ab_0)^\top \Wb^\top %
    =(\ab_e-\ab_0)^\top \Wbt^\top %
    \Jb_\hb(\zb).
\end{equation}

Let $\Ab\in\RR^{K\times M}$ be the matrix with columns $\ab_e-\ab_0$ for $e\in[M]$. Then, for all $\zb$:
\begin{equation}
\label{eq:jacobian}
    \Ab^\top \Wb^\top %
    =\Ab^\top \Wbt^\top %
    \Jb_\hb(\zb).
\end{equation}

Differentiating once more w.r.t.\ $\zb$ yields for all $\zb$:
\begin{equation}
    \label{eq:hessian}
    \bm 0 = \Ab^\top \Wbt^\top %
    \Hb_\hb(\zb),
\end{equation}
where the 3-tensor
$\Hb_\hb(\zb)\in\RR^{\dimZ\times \dimZ\times \dimZ}$ denotes the Hessian of $\hb$, i.e., for all $i,j\in[\dimZ]$ and all $\zb$
\begin{equation}
\label{eq:second_order_partials}
    \bm 0 = \Ab^\top \Wbt^\top 
    \frac{\partial^2}{\partial z_i \partial z_j}\hb(\zb).
\end{equation}
By assumption \textit{(iii)}, the matrix $\Ab^\top\Wbt^\top$ has full column rank and thus a left inverse, i.e., there exists $\Vb\in\RR^{\dimZ\times M}$ such that $\Vb\Ab^\top\Wbt^\top=\Ib_{\dimZ}$.
Multiplying~\eqref{eq:second_order_partials} on the left by $\Vb$ then yields for all $i,j\in[\dimZ]$ and all $\zb$:
\begin{equation}
    \frac{\partial^2}{\partial z_i \partial z_j}\hb(\zb)=\bm 0.
\end{equation}

This implies that $\hb$ must be affine, i.e., there exist $\Mb\in\RR^{\dimZ\times\dimZ}$ and $\bb\in\RR^\dimZ$ such that for all $\zb$:
\begin{equation}
\label{eq:affine_transformation_h}
    \hb(\zb)=\Mb\zb+\bb.
\end{equation}

Further, since $\hb$ is invertible, $\Mb$ must be invertible.

Recall that $\Zbt_e\overset{d}{=}\hb(\Zb_e)=\Mb\Zb_e+\bb$. It follows from~\eqref{eq:ze}, \eqref{eq:zetilde} and assumption \textit{(ii)} that for all $e\in[M]_0$:
\begin{equation}
    \Ncal\left(\Wbt(\ab_e-\ab_0), \Ib\right)
    =\Ncal\left(\Mb\Wb(\ab_e-\ab_0)+\bb, \Mb\Mb^\top
    \right).
\end{equation}

By equating the covariances, we find
\begin{equation}
    \Mb\Mb^\top=\Ib, 
\end{equation}
i.e., $\Mb=\Ob$ for some orthogonal matrix $\Ob\in O(\dimZ)$.

Similarly, by equating the means, we obtain for all $e$:
\begin{equation}
\label{eq:constraint_on_W}
    \Wbt(\ab_e-\ab_0)=\Ob\Wb(\ab_e-\ab_0)+\bb %
\end{equation}
For $e=0$, \eqref{eq:constraint_on_W} yields 
\begin{equation}
    \bb=\bm 0.
\end{equation} 
Stacking~\eqref{eq:constraint_on_W} for all other $e\in[M]$, we obtain
\begin{equation}
\label{eq:general_form_W}
    \Wbt\Ab = \Ob\Wb\Ab.
\end{equation}
This completes the proof.
\end{proof}

\subsection{Proof of~\texorpdfstring{\cref{cor:id_gaussian}}{}}
\label{app:proof_corollary_id_gaussian}
\corollaryidgaussian*
\begin{proof}
If $\rank(\Ab)=K$, then $\Ab$ has a right inverse, i.e., there exists $\Kb\in\RR^{M\times K}$ such that $\Ab\Kb=\Ib_K$. 
Right multiplication of~\eqref{eq:general_form_W} by $\Kb$ then yields
\begin{equation}
\label{eq:special_form_W}
    \Wbt=\Ob\Wb.
\end{equation}
\end{proof}

\subsection{Proof of~\texorpdfstring{\cref{thm:extrapolation}}{}}
\label{app:proof_extrapolation}
\extrapolation*
\begin{proof}
With $\hb=\dect^{-1}\circ\dec$, the condition in~\eqref{eq:equal_test_prediction} is equivalent to
\begin{equation}
\label{eq:equivalent_condition}
    \hb(\Zb+\Wb\ab_\te)\overset{d}{=}\Zbt+\Wbt\ab_\te.
\end{equation}
By~\cref{thm:affine_identifiability_gaussian}, we have $\hb(\zb)=\Ob\zb$ for some orthogonal matrix $\Ob\in O(d_Z)$.

Together with $\Zb\sim\Ncal(-\Wb\ab_0,\Ib)$, this lets us compute the distribution of 
the LHS of~\eqref{eq:equivalent_condition} as:
\begin{equation}
\label{eq:LHS_distribution}
\hb(\Zb+\Wb\ab_\te)=\Ob(\Zb+\Wb\ab_\te)\sim\Ncal\left(-\Ob\Wb\ab_0+\Ob\Wb\ab_\te,\Ib\right).
\end{equation}
Similarly, since $\Zbt\sim\Ncal\left(-\Wbt\ab_0,\Ib\right)$, the distribution of the RHS of~\eqref{eq:equivalent_condition} is given by
\begin{equation}
\label{eq:RHS_distribution}
\Zbt+\Wbt\ab_\te\sim \Ncal\left(-\Wbt\ab_0+\Wbt\ab_\te,\Ib\right).
\end{equation}
To complete the proof, it thus remains to show that the means of~\eqref{eq:LHS_distribution} and~\eqref{eq:RHS_distribution} are equal, i.e., 
\begin{equation}
\label{eq:equal_means}
    \Ob\Wb\left(\ab_\te-\ab_0\right)=\Wbt\left(\ab_\te-\ab_0\right)
\end{equation}
Next, by~\eqref{eq:span}, i.e., the assumption that $(\ab_\te-\ab_0)$ lies in the span of $\{\ab_e-\ab_0\}_{e\in[M]}$, there exists $\alphab\in\RR^M$ such that
\begin{equation}
\label{eq:linear_combination}
\textstyle
    \ab_\te-\ab_0 = \sum_{e\in[M]}\alpha_e(\ab_e-\ab_0)=\Ab\alphab
\end{equation}
where, as before, $\Ab\in\RR^{K\times M}$ is the matrix with columns $(\ab_e-\ab_0)$.
Substituting~\eqref{eq:linear_combination} into~\eqref{eq:equal_means} yields
\begin{align}
\label{eq:second_last_step}
    \Ob\Wb\Ab\alphab=\Wbt\Ab\alphab
\end{align}
Finally, it follows from~\cref{thm:affine_identifiability_gaussian} that 
\begin{equation}
\label{eq:AWLWL}
    \Ob\Wb\Ab=\Wbt\Ab.
\end{equation}
which upon substitution into~\eqref{eq:second_last_step} 
completes the proof.
\end{proof}

\subsection{Proof of~\texorpdfstring{\cref{prop:sems_special_case}}{}}
\sems*
\begin{proof}
The distribution of $\Zb$ induced by the SEM is most easily understood via the \textit{reduced form} expression
\begin{equation}
\label{eq:linear_SCM_reduced_form}
    \Zb=(\Ib-\Bb^\top)^{-1}\etab
\end{equation}
where the inverse $(\Ib-\Bb^\top)^{-1}$ exists since, by assumption, $\rho(\Bb)<1$.
The observational (unintervened) distribution is thus given by $\PP_\Zb=(\Ib-\Bb^\top)^{-1}_\#\PP_{\etab}$. 
Now consider the shift interventions from~\eqref{eq:SEM_shift_interventions}.
Analogous to~\eqref{eq:linear_SCM_reduced_form}, the reduced form of~\eqref{eq:SEM_shift_interventions} is given by
\begin{equation}
\label{eq:linear_SCM_shift_intervention_reduced_form}
\Zb%
=(\Ib-\Bb^\top)^{-1}\etab +(\Ib-\Bb^\top)^{-1}\ab_e.
\end{equation}
which is equal in distribution to
\begin{equation}
    \Zb^\pert=\Zb^\base+\Wb\ab_e
\end{equation}
with $\Zb^\base\sim\PP_\Zb=(\Ib-\Bb^\top)^{-1}_\#\PP_{\etab}$ and $\Wb=(\Ib-\Bb^\top)^{-1}$.
\end{proof}
\begin{remark}
    The correspondence between mean shift perturbations and shift interventions appears to only hold  for \textit{linear} SCMs (at least in this simple form). Consider, for example, a nonlinear additive noise model,
    \begin{equation}
    \label{eq:nonlinear_additive_noise_model}
        \Zb:=\fb(\Zb)+\etab,
    \end{equation}
    with reduced form given by
        $\Zb=\gb(\etab)$,
    where $\gb$ is the inverse of the mapping $\zb\mapsto \zb-\fb(\zb)$.
    For shift interventions in~\eqref{eq:nonlinear_additive_noise_model} to match our perturbation model, we then must have
    \begin{equation}
        \Zb_e=\gb(\etab+\ab_e)= \gb\left(\gb^{-1}(\Zb)+\ab_e\right)=\phib\left(\Zb,\ab_e\right)
    \end{equation}
    for suitable $\phib$ and $\ab_e$.
    Thus, if $\fb$ is nonlinear, so is $\gb$ and this implies under weak assumptions that $\phib$ is nonlinear, too. In this case, shift interventions in a nonlinear SCM do not, in general, amount to mean shift perturbations (i.e., linear~$\phib$). 
\end{remark}

\section{PDAE training: Additional losses and complete algorithm}
\label{app:algorithm}
In this appendix, we discuss additional loss compoments (other than the perturbation loss from the main text), which are not required based on our theoretical results, but which may still benefit empirical performance.

\textbf{Reconstructions.}
When considering the synthetic observations from~\eqref{eq:synthetic_observations} for $h=e$, i.e., when evaluated on the estimated perturbed latents $\zbh^\pert_{e,i}$, we refer to the corresponding (random) decoder outputs as \textit{reconstructions},
\begin{equation}
\label{eq:reconstructions}
\Xbh_{e,i}=\dech\big(\zbh^\pert_{e,i}, \,\epsilonb%
\big) \qquad \text{where} \qquad \epsilonb%
\sim\QQ_\epsilonb.
\end{equation}

\textbf{Conditional distribution over reconstructions.}
For any given observation $\xb$, we can sample reconstructions from the stochastic decoder by fixing its first argument to the encoding $\ench(\xb)$
and sampling noise from $\QQ_\epsilonb$ as in~\eqref{eq:reconstructions}. 
Together, the encoder, decoder, and noise distribution thus induce a distribution over reconstructions, formally given by the push-forward 
\begin{equation}
\label{eq:conditional_reconstruction_distribution}
\dech\left(\ench(\xb),\cdot\right)_\#\QQ_\epsilonb.
\end{equation}
Since this distribution should ideally match that of $\Xb|\ench(\Xb)=\ench(\xb)$ (i.e., the true distribution of observations that share the same encoding),\footnote{\looseness-1 For example, if $\ench$ is injective, the true reconstruction distribution is a point mass on $\xb$,  whereas if $\ench$ is constant, it is the full (unconditional) distribution of $\Xb$. 
In this sense, the distribution of reconstructions captures the information about $\xb$ that is left unexplained given only its encoding~\citep{shen2024distributional}.} we denote it by
 $\PPh_{\Xb|\ench(\Xb)=\ench(\xb)}$. 
Note that the first argument of $\dech$ in~\eqref{eq:conditional_reconstruction_distribution} is considered fixed, and all stochasticity stems from the noise~$\epsilonb$.
This contrasts with the simulated perturbed distributions $\PP_{e\to h}$ in~\eqref{eq:induced_synthetic_distribution}, where both $\Xb_e$ and $\epsilonb$ are random.   

\textbf{Conditional reconstruction loss.} 
The conditional reconstruction loss $\Lcal_\text{cond-rec}$ is given by the sum of domain-specific DPA losses~\citep{shen2024distributional},
\begin{equation}
\begin{aligned}
\label{eq:reconstruction_loss_main}
    &\Lcal_\text{cond-rec}\left(\ench,\dech; \left\{\PP_e\right\}_{e\in[M]_0}\right)
    :=\sum_{e\in[M]_0} 
    -\EE_{\Xb_e\sim\PP_e}
    \left[\ES_\beta
    \left(\PPh_{\Xb|\ench(\Xb)=\ench(\Xb_e)}, \Xb_e
    \right)
    \right],
\end{aligned}
\end{equation}
\looseness-1 where each summand on the RHS of~\eqref{eq:reconstruction_loss_main} involves a \textit{conditional} energy score (i.e., one whose first argument also depends on $\Xb_e$), see~\cref{sec:representation_learning} for further discussion.
It should be noted that the pairwise sum in the perturbation loss of~\eqref{eq:combined_loss} also includes the case $h=e$.  Enforcing $\PP_e=\PPh_{e\to e}$ can be viewed as a marginal reconstruction objective and thus as a weaker constraint than that enforced by~$\Lcal_\text{cond-rec}$.

\begin{figure}[t]
    \centering
    \includegraphics[width=\linewidth]{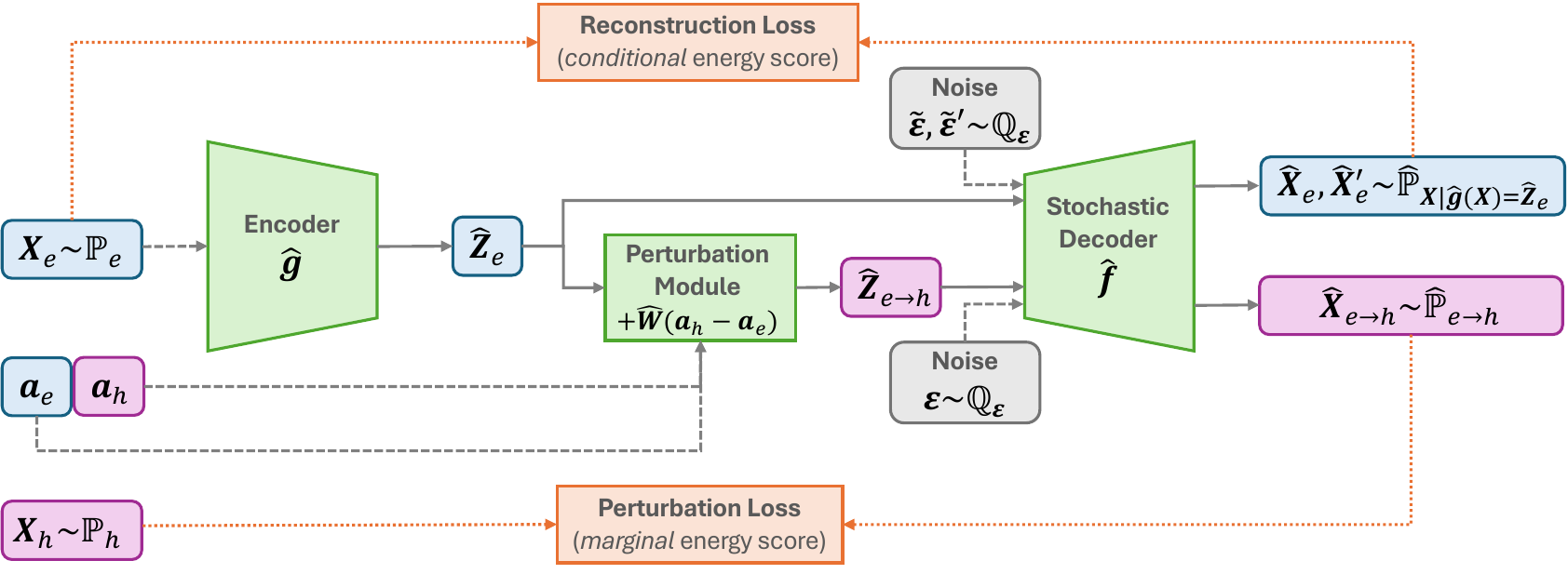}
    \caption{\textbf{Illustration of the full PDAE architecture including the reconstruction loss.} The distributional reconstruction loss \textit{(top)}, measures, for each source domain $e$, the dissimilarity between the (empirical) 
     true conditional distribution %
     of source observations~$\Xb_e$ that are mapped to the same encoding $\Zbh_e$ and the corresponding distribution %
     induced by the stochastic decoder.}
    \label{fig:pdae_full}
\end{figure}

\textbf{Prior loss.} To turn PDAE into a full generative model and remove unnecessary ambiguities from the learned representation, we propose an additional loss which quantifies deviations of the encoded latent basal state distributions from a fixed Gaussian prior.
Since, by~\cref{lemma:mean_covariance_ambiguity_no_intercept}, the prior mean and covariance are arbitrary, we choose $\PP^\base=\Ncal(\bm0,\Ib)$ for convenience.
Denote by $\PPh_{\ench(\Xb_e)-\Wb\ab_e}$ the distribution of
\begin{equation}
\ench(\Xb_e)-\Wbh\ab_e, \qquad \text{where} \qquad \Xb_e\sim\PP_e.
\end{equation}
We then define the prior loss as the sum over domains of the negative expected energy score between the true fixed prior and the encoded basal state distribution from domain $e$,
\begin{equation}
    \Lcal_\text{prior}\left(\ench,\Wbh;\left\{\left(\PP_e,\ab_e\right)\right\}_{e=0}^M\right):=\sum_{e\in[M]_0}-\EE_{\Zb^\base\sim\PP^\base}\left[\ES_\beta\left(\PPh_{\ench(\Xb_e)-\Wbh\ab_e},\Zb^\base\right)\right].
\end{equation}
In practice, the prior loss can be approximated using the estimated latent basal states,
\begin{equation}
    \label{eq:estimated_basal_states}
    \zbh^\base_{e,i}:=\zbh^\pert_{e,i}-\Wbh\ab_e=\ench(\xb_{e,i})-\Wbh\ab_e
\end{equation}

\textbf{Sparsity.} \looseness-1 Similar to~\citet{bereket2023modeling}, one could also consider a sparsity penalty $\norm{\Wbh}_{2,1}$ on the inferred $\Wbh$. (We only include this for completeness; the present work does not explore sparsity.)

\textbf{Combined loss.} Given loss weight hyperparameters $(\lambda_\textsc{r},\lambda_\text{prior},\lambda_\textsc{s})$ and denoting the perturbation loss from~\eqref{eq:combined_loss} by $\Lcal_\pert\left(\ench,\dech,\Wbh; \left\{\left(\PP_e,\ab_e\right)\right\}_{e=0}^M\right)$ we can then combine the different losses into one objective as follows:
\begin{equation}
\begin{aligned}
    \label{eq:combined_losses}
    \left(\ench^*,\dech^*,\Wbh^*\right) \in \argmin_{\ench,\dech,\Wbh}\, 
    &\Lcal_\pert\left(\ench,\dech,\Wbh; \left\{\left(\PP_e,\ab_e\right)\right\}_{e=0}^M\right)+\lambda_\textsc{r} \Lcal_\text{cond-rec}\left(\ench,\dech; \left\{\PP_e\right\}_{e=0}^M\right)
    \\
    &+ \lambda_\text{prior}\Lcal_\text{prior}\left(\ench,\Wbh;\left\{\left(\PP_e,\ab_e\right)\right\}_{e=0}^M\right) +\lambda_\textsc{s}\norm{\Wbh}_{2,1}.
\end{aligned}
\end{equation}
\looseness-1 To avoid that the encoder captures latent dimensions of maximal variation as in DPA~\citep{shen2024distributional} rather than focus on the perturbation-relevant latents $\Zb$, we do not update the encoder parameters using the reconstruction loss.

\textbf{Algorithm.} The complete algorithm for one training iteration for PDAE is provided in~\cref{alg:training}.

\begin{algorithm}[tbp]
    \caption{PDAE training step}
    \label{alg:training}
    \begin{algorithmic}[1]
    \doublespacing
        \REQUIRE mini-batch $\{(\xb_i, \ab_i)\}_{i=1}^B$, model $\left(\ench_\phib,\dech_\thetab, \Wbh\right)$, learning rates $(\eta_\theta,\eta_\phi, \eta_\textsc{w})$, loss weights $(\lambda_\textsc{r},\lambda_\text{prior},\lambda_\textsc{s})$
        \STATE $\Lcal_\text{cond-rec}(\phib,\thetab),\Lcal_\pert\left(\phib,\thetab,\Wbh\right)\gets 0,0$ \COMMENT{initialize losses}
        \FOR{$i=1, \dots, B$}
            \STATE $\zbh^\pert_i \gets \ench_\phib(\xb_i)$ 
            \COMMENT{encode}
            \STATE $\epsilonb_i,\epsilonb_i'\sim \Ncal(\bm 0, \Ib)$ 
            \COMMENT{noise for reconstructions}
            \STATE $\xbh_i, \xbh_i' \gets \dech_\thetab\left(\zbh^\pert_i,\epsilonb_i\right),\dech_\thetab\left(\zbh^\pert_i,\epsilonb_i'\right)$
            \COMMENT{noisy reconstructions}
            \STATE $\Lcal_\text{cond-rec}(\phib,\thetab)\pluseq \frac{1}{2B}\left(\norm{\xb_i-\xbh_i} +\norm{\xb_i-\xbh_i'} -\norm{\xbh_i-\xbh_i'} \right)$             
            \COMMENT{conditional energy score}
            \STATE $\zbh^\base_i\gets \zbh_i^\pert -\Wbh\ab_i$             \COMMENT{estimated basal states}
            \FOR{$j\in[B]\setminus \{i\}$}
                \STATE $\zbh^\pert_{j\to i}\gets \zbh_j^\base+\Wbh\ab_i$             \COMMENT{synthetic perturbed latents}

                \STATE $\epsilonb_{ji}\sim\Ncal(\bm 0, \Ib)$
                \COMMENT{noise for synthetic observations}
                \STATE $\xbh_{j\to i}\gets \dech_\theta\left(\zbh^\pert_{j\to i},\epsilonb_{ji}\right)$ 
                \COMMENT{synthetic observations}
            \ENDFOR
            \STATE $\Lcal_\pert\left(\phib,\thetab,\Wbh\right)\pluseq \frac{1}{B(B-1)}\sum_{j\in[B]\setminus\{i\}} \left[\norm{\xb_i-\xbh_{j\to i}} -\frac{1}{B-2}\sum_{i\neq j'<j}\norm{\xbh_{j\to i}-\xbh_{j' \to i}} \right]$
            \STATE $\xib_i\sim\Ncal\left(\bm0, \Ib\right)$
            \COMMENT{noise for prior loss}
        \ENDFOR
        \STATE $\Lcal_\text{prior}\left(\phib,\Wbh\right)\gets 
        \frac{1}{B^2}\sum_{i,j\in[B]} \norm{\xib_i-\zbh^\base_j}-\frac{1}{B(B-1)}\sum_{j'<j\in[B]}\norm{\zbh^\base_j-\zbh^\base_{j'}}$
        \COMMENT{marginal energy score}
        \STATE $\thetab \gets \thetab - \eta_\theta\nabla_\thetab \left(\Lcal_\pert\left(\phib,\thetab,\Wbh\right)+\lambda_\textsc{r}\Lcal_\text{cond-rec}(\phib,\thetab)\right)$
        \COMMENT{update decoder}
        \STATE $\phib \gets \phib - \eta_\phi\nabla_\phib \left(\Lcal_\pert\left(\phib,\thetab,\Wbh\right)+\lambda_\text{prior}\Lcal_\text{prior}\left(\phib,\Wbh\right)\right)$
        \COMMENT{update encoder}
        \STATE $\Wbh \gets \Wbh - \eta_\textsc{w}\nabla_{\Wbh}\left(\Lcal_\pert\left(\phib,\thetab,\Wbh\right)+\lambda_\text{prior}\Lcal_\text{prior}\left(\phib,\Wbh\right)+\lambda_\textsc{s}\norm{\Wbh}_{2,1}\right)$
    \end{algorithmic}
\end{algorithm}

\textbf{Code.} We will publicly release PDAE code upon publication.

\clearpage
\section{Additional details and results for real-world gene perturbation data}
\label{app:additional_details_and_results_gene_data}
\textbf{Hyperparameter choices.} Following~\citet{miller2025deep}, we use the hyperparameters prescribed in the respective publications (GEARS, PRESAGE, scLambda). For PDAE, motivated by our theory, we use a latent space of size $d_Z = K$ with $K$ the number of single perturbations in each dataset (105 for Norman19 and 28 for Wessels23). For both experiments, we use $\beta=1$, a reconstruction loss weight of $\lambda_\textsc{r}=0.05$, a prior loss weight of $\lambda_\text{prior}=0.0001$, no sparsity ($\lambda_S=0$), and an encoder and decoder with [2048, 512] and [512, 2048] hidden units, respectively, with ELU activation and softplus output layer. These choices were selected from a handful of manual runs on Norman19 (i.e., no extensive hyperparameter tuning was carried out). We also use the validation set for early stopping based on the perturbation energy loss.

The results for the same experiment reported in the main paper but on the data from~\citet{wessels2023efficient} are shown in~\cref{fig:results_wessels}, see the caption for details.

\begin{figure}[tbp]
    \centering
    \includegraphics[height=0.35\linewidth, trim={0 0 0em 5em}, clip]{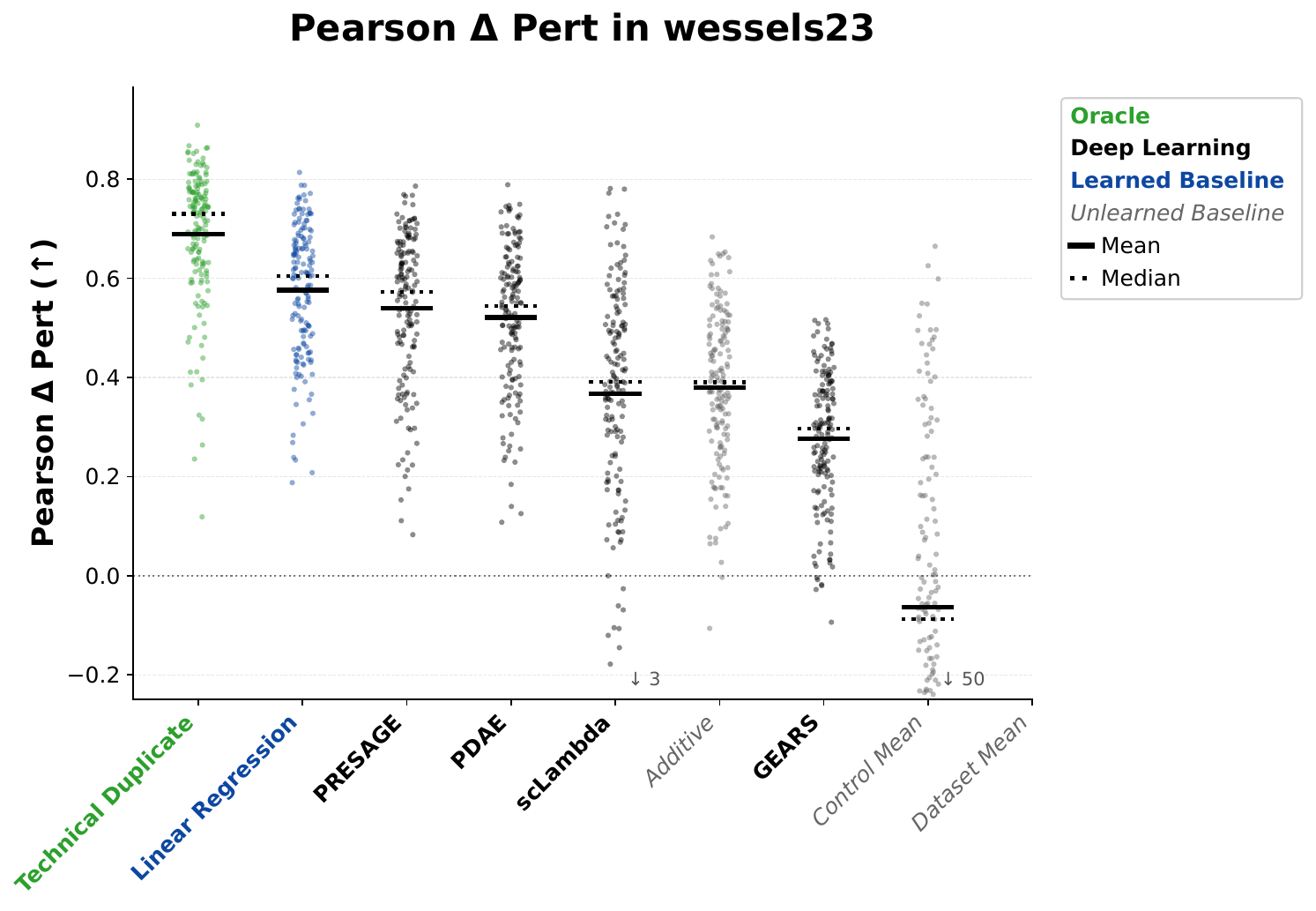}%
    \includegraphics[height=0.35\linewidth, trim={0 0 15em 5em}, clip]{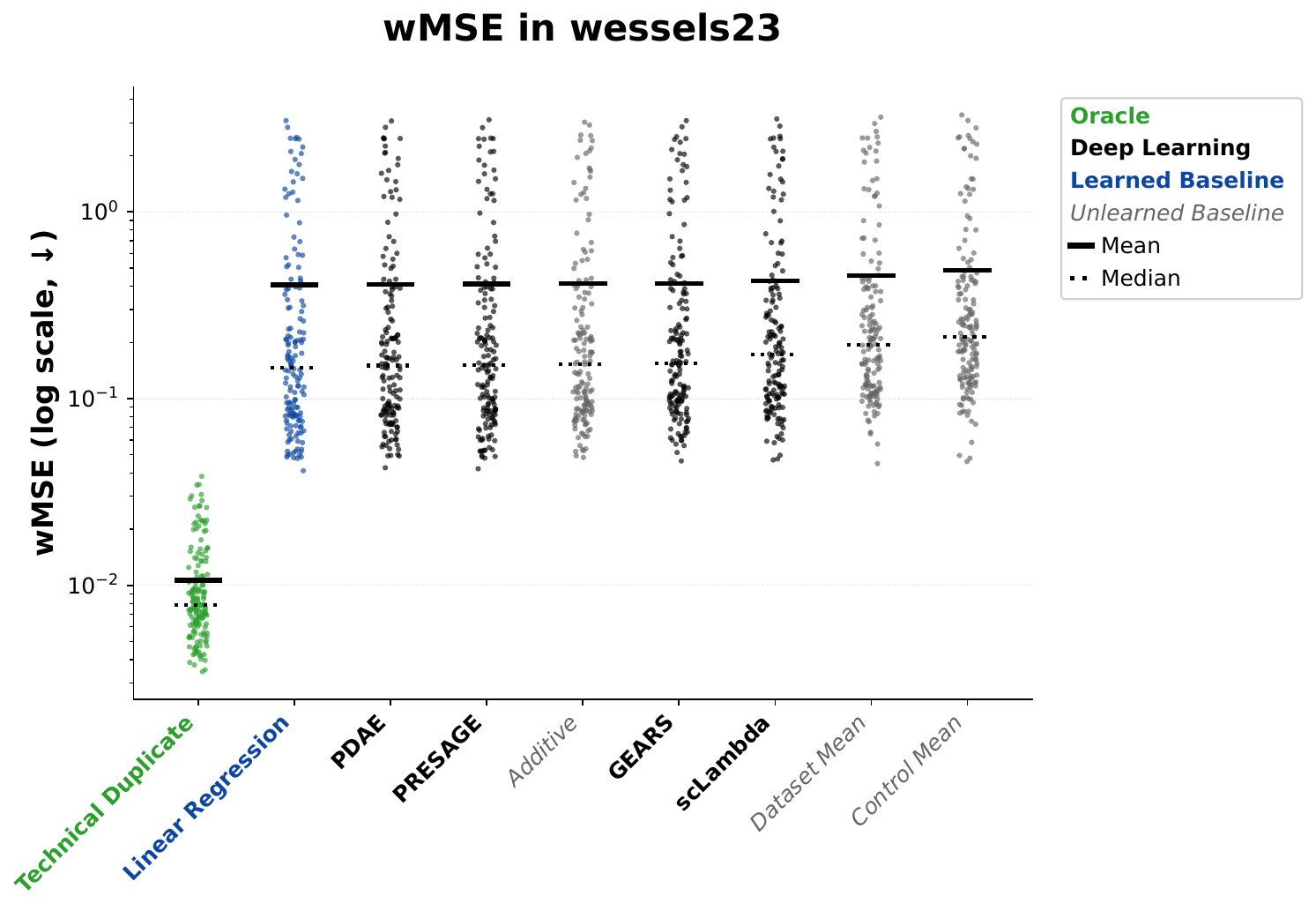}
    \caption{\small \looseness-1 \textbf{Results for gene perturbation prediction.} Shown are wPearson$\Delta$Pert (left) and wMSE (right) on the CRISPRi combinatorial PerturbSeq data from~\citet{wessels2023efficient}. Each dot in a strip-plot corresponds to a different test double perturbation. Due to the large spread, methods are sorted by median.
    In terms of wPearson$\Delta$Pert, PDAE performs better than baselines, GEARS, and scLambda, and comparable to or slightly worse than PRESAGE and its linear instantiation.  In terms of wMSE (right), all methods perform similarly bad and---in contrast to the \citet{norman2019exploring} data---the performance gap to the oracle is very large.} 
    \label{fig:results_wessels}
\end{figure}

\section{Additional details and results for simulations on synthetic data}
\label{app:experimental_details}
Here, we provide a more detailed account and additional results for the simulation study on synthetic data. 
We first summarize the main points in~\cref{app:summary_synthetic_data} and defer additional details to the following subsections.

\subsection{Summary}
\label{app:summary_synthetic_data}

\textbf{Methods.}
We compare our approach with CPA~\citep{lotfollahi2023predicting} and the following baselines: 
\begin{itemize}[leftmargin=*]
    \item \textit{Pool all}: output the distribution and mean obtained by pooling all training observations;%
    \item \textit{Pseudobulking}: like \textit{Pool All} but pool only training data arising from individual perturbations involved in the combination to be predicted
    (e.g., pool data from $(1,0,0)^\top$ and $(0,1,0)^\top$ to predict $\ab_\te=(1,1,0)^\top$);%
    \item \looseness-1 \textit{Linear regression}: linearly regress the domain-specific means of observations $\mub^{\xb}_e\in\RR^\dimX$ on $\ab_e$ and use the resulting model to predict the test mean $\mub_\te^{\xb}$ from~$\ab_\te$.
\end{itemize}
\looseness-1 The former two were also used by~\citet{lotfollahi2023predicting}; we propose the latter as an additional (stronger) baseline.
Since CPA and linear regression only predict a mean, in these cases, we apply the test mean shift to all observations in the reference condition to simulate the test distribution and use this to compute distributional similarity metrics.

\textbf{Metrics.}
\looseness-1 
To assess the distributional fit between predicted and true (empirical) test distributions, we use the energy distance~\citep[ED;][]{szekely2013energy}, i.e., twice the normalized negative expected (marginal) energy score, and the maximum mean discrepancy~\citep[MMD;][]{gretton2012kernel} with Gaussian kernel and bandwidth chosen by the median heuristic, see~\cref{sec:distributional_similarity} for more on measures of distributional similarity.
Since some methods are primarily designed for  mean prediction, we also report the L2 norm of the difference between the predicted and true mean, $\norm{\mub^\xb-\widehat{\mub}^\xb}$.

\textbf{Data.}
\looseness-1 
For ease of visualization, we first consider the case of $\dimZ\!=\!\dimX\!=\!2$-dimensional latents and observations.
The base distribution $\PP_\Zb$ is a zero-mean, isotropic Gaussian with standard deviation~$\sigma\!=\!0.25$.
We consider $K\!=\!3$ elementary perturbations with associated shift vectors $\wb_1\!=\!(1,0)^\top$, $\wb_2\!=\!(0,1)^\top$, and $\wb_3\!=\!(1,1)^\top$ and create $M\!+\!1\!=4$ training domains with perturbation labels $\ab_0\!=\!(0,0,0)^\top$, $\ab_1\!=\!(1,0,0)^\top$, $\ab_2\!=\!(0,1,0)^\top$, and $\ab_3\!=\!(0,0,1)^\top$. In this case, the sufficient diversity condition~\textit{(iii)} used to establish identifiability in~\cref{thm:affine_identifiability_gaussian} is satisfied.
For testing, we distinguish two scenarios: 
in-distribution (ID) and out-of-distribution (OOD) test cases, depending on whether the 
perturbed latents resulting from $\ab_\te$ lie mostly within or outside the support of the empirical distributions of perturbed training latents, 
respectively. 
We consider 14 randomly generated ID and OOD test cases each, see~\cref{fig:dgp_gt_random_labels} and \cref{app:experimental_details} for a visualization and further details. 
To generate observations, we use the complex exponential $\xb\!=\!\dec(\zb)\!=\!\mathrm{e}^{z_1}(\cos z_2,\sin z_2)$ as a deterministic nonlinear mixing function.

\textbf{Experimental details.}
\looseness-1
We generate $N_e=2^{14}$ observations for each domain. Both CPA and PDAE use $4$-hidden-layer MLPs with $64$ hidden units as encoders and decoders, a $\widehat\dimZ=2$-dimensional latent space, and are trained for $2000$ epochs using a batch size of $2^{12}$.
For PDAE, we set the energy score exponent $\beta\!=\!1$ 
and use a learning rate of $0.005$.
\setstretch{0.99}
For CPA, we choose a $1$-hidden-layer MLP with $64$ hidden units for the adversarial classifier. For the remaining hyperparameters, we randomly sample $100$ joint configurations and choose the one with the smallest MMD across $7$ ID test cases, see%
~\cref{app:experimental_details} for details.
We report results for average performance across the remaining $7$ ID and $14$ OOD test cases over $5$ repetitions of randomly generating data and initializing, training, and evaluating models. 

\begin{table}[tbp]
\caption{
\small
\looseness-1
\textbf{Results on synthetic 2D data.}
We report average performances over 7 ID and 14 OOD test cases; shown are the mean $\pm$ one standard deviation over 5 random seeds. Best results are highlighted in bold. Theoretical guarantees for PDAE rely on distributional overlap in the latent space (ID); no method is expected to perform well on OOD test cases.
}
\small
\centering
\label{tab:results}
 \resizebox{\textwidth}{!}{
\begin{tabular}{lcccccc}
\toprule
\multirow{2}{*}{\textbf{Method}} 
& \multicolumn{3}{c}{\textbf{In-Distribution (ID) Test}} &  \multicolumn{3}{c}{\textbf{Out-of-Distribution (OOD) Test}}\\
\cmidrule(lr{0.5em}){2-4} \cmidrule(lr{0.5em}){5-7}
& \textbf{ED} ($\downarrow$) & \textbf{MMD} ($\downarrow$) & $\norm{\mub_\xb-\widehat{\mub}^\xb}_2$ ($\downarrow$)
& \textbf{ED} ($\downarrow$) & \textbf{MMD} ($\downarrow$) & $\norm{\mub_\xb-\widehat{\mub}^\xb}_2$ ($\downarrow$)
\\
\midrule
Pool All & 0.855 \scriptsize{$\pm$ .005} & 0.557 \scriptsize{$\pm$ .004} & 0.441 \scriptsize{$\pm$ .001} & 1.908 \scriptsize{$\pm$ .004} & 1.266 \scriptsize{$\pm$ .004} & 1.298 \scriptsize{$\pm$ .002} \\
Pseudobulking & 0.596 \scriptsize{$\pm$ .004} & 0.716 \scriptsize{$\pm$ .006} & 0.382 \scriptsize{$\pm$ .002} & 1.130 \scriptsize{$\pm$ .005} & 0.965 \scriptsize{$\pm$ .006} & 0.894 \scriptsize{$\pm$ .004} \\
Linear Regression & 0.069 \scriptsize{$\pm$ .001} & 0.181 \scriptsize{$\pm$ .003} & 0.131 \scriptsize{$\pm$ .003} & 0.536 \scriptsize{$\pm$ .010} & 0.808 \scriptsize{$\pm$ .012} & 0.524 \scriptsize{$\pm$ .007} \\
\midrule
CPA & 0.600 \scriptsize{$\pm$ .239} & 1.152 \scriptsize{$\pm$ .287} & 0.516 \scriptsize{$\pm$ .165} & 1.538 \scriptsize{$\pm$ .261} & 1.885 \scriptsize{$\pm$ .241} & 1.132 \scriptsize{$\pm$ .171} \\
PDAE \textbf{(Ours)} & \textbf{0.001 \scriptsize{$\pm$ .000}} & \textbf{0.003 \scriptsize{$\pm$ .001}} & \textbf{0.012 \scriptsize{$\pm$ .003}} & \textbf{0.018 \scriptsize{$\pm$ .006}} & \textbf{0.029 \scriptsize{$\pm$ .007}} & \textbf{0.090 \scriptsize{$\pm$ .017}} \\
\bottomrule
\end{tabular}
 }
\end{table}

\textbf{Results.}
\looseness-1 
The results are shown in~\cref{tab:results}.
For the ID test setting, PDAE performs best, achieving a near-perfect distributional fit. 
CPA outperforms the pooling and pseudobulking baselines (consistent with results reported by~\citet{lotfollahi2023predicting} on gene perturbation data), but does substantially worse than linear regression and PDAE at both mean and distributional prediction.
For the OOD test setting, 
all methods expectedly perform worse, %
with PDAE yielding the least bad performance,\footnote{\looseness-1 We believe that PDAE benefits from implicit regularization in the network and that worse counterexamples can be constructed.}
see~\cref{fig:qualitative_results_main} and its caption for details.

\textbf{Noisy observations.}
We additionally consider a modified version of the aforementioned experiment with $\dimX=10$ (instead of $\dimX=2$) observed dimensions by concatenating $\dimeps=8$ additional dimensions of mean-zero Gaussian noise and investigate the behavior for increasing noise variance (i.e., decreasing signal-to-noise ratio), see~\cref{app:robustness_experiment} for details.
As shown in~\cref{fig:qualitative_results_main}~\textit{(Right)},
 PDAE outperforms other methods for low to medium noise levels, whereas for high noise levels, all methods except CPA exhibit similar performance.

\begin{figure}[t]
    \centering
    \includegraphics[width=\textwidth]{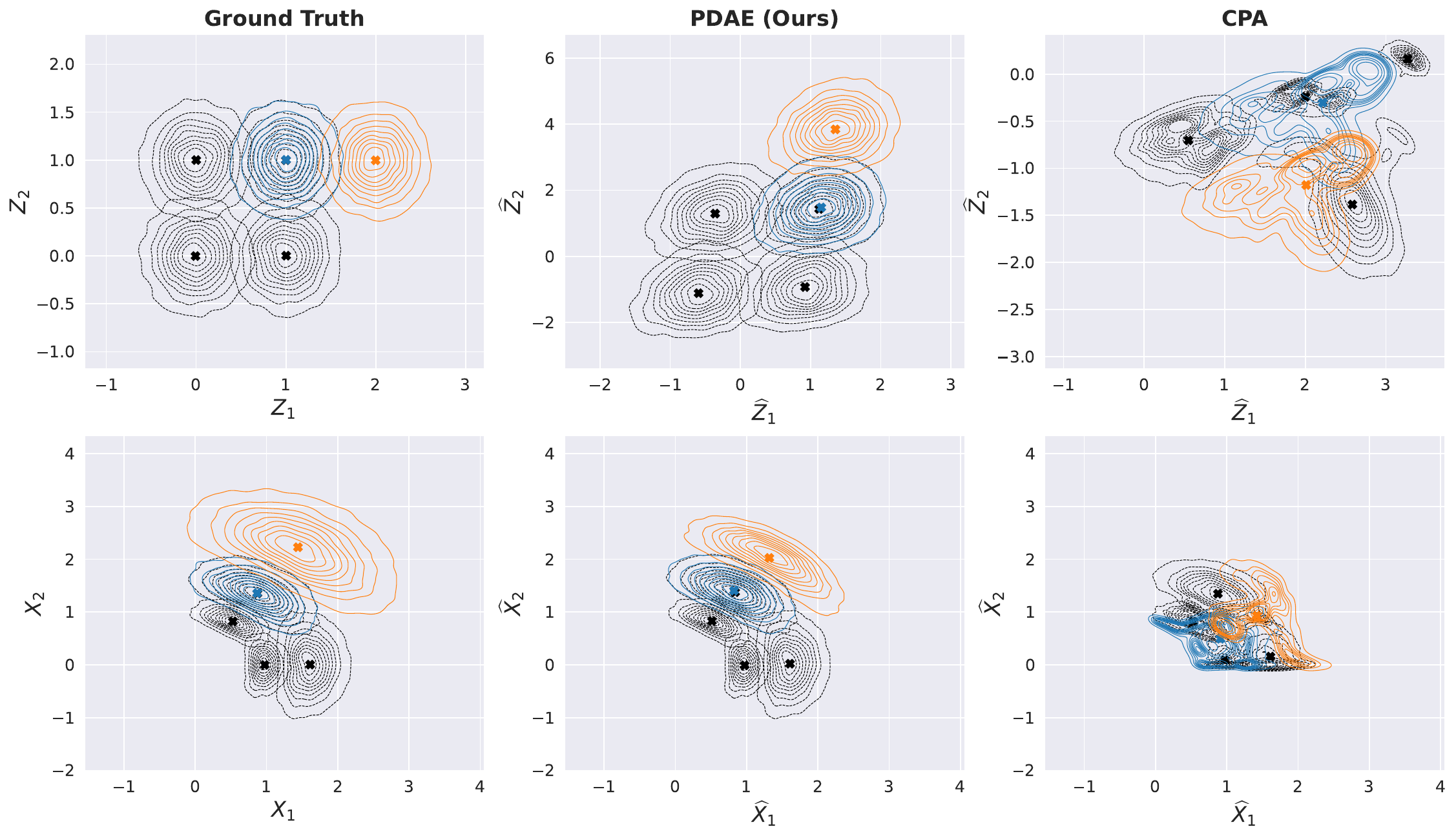}
    \caption{\small \looseness-1 
    \textbf{Qualitative comparison of PDAE and CPA on synthetic 2D data.}
    Rows correspond to latent space (top) and observation space (bottom).  Columns show the ground truth data (left), PDAE predictions (center), and CPA predictions (right).
    Training domains are shown in grey, an in-distribution (ID) test case with $\ab^\textsc{id}_\te\!=\!(1,1,0)^\top$ (which overlaps with one of the training domains) in blue, and an out-of-distribution (OOD) test case with $\ab_\te^\textsc{ood}\!=\!(1,0,1)^\top$ in orange. 
    All plots show kernel density estimates of the distributions; crosses (x) indicate the corresponding means.
    As can be seen, PDAE recovers an affine transformation of the true latents (top, center) leading to accurate distributional predictions for the training and ID test domain (bottom, center). 
    However, the OOD test domain is mapped to a part of the latent space not seen during training (top, center). As a result, the corresponding decoder output does not accurately match the true OOD distribution (bottom left vs center).
    CPA accurately predicts the means of the training distribution (bottom right, black crosses) but does not recover the true latents up to an affine transformation
(top right). 
    Moreover, the predicted distributions for CPA do not match the ground truth particularly well, particularly for the test conditions (bottom left vs right). However, recall that---unlike PDAE---CPA is not trained for distributional reconstruction, see~\cref{sec:CPA} for details.
}
    \label{fig:qualitative_results}
\end{figure}

\begin{table}[tbp]
\centering
\caption{
\textbf{Tuned Hyperparameters of CPA with their search space and optimal value.} We sample 100 random configurations from the joint search space and train a CPA model for 2000 epochs for each. The optimal configuration minimizes the MMD on ID validation data.}
\label{tab:simdata:cpa:hyperparameters}
\begin{tabular}{@{}lll@{}}
\toprule
\textbf{Hyperparameter} & \textbf{Search Space} & \textbf{Selected Optimal Value} \\
\midrule
\texttt{step\_size\_lr}     & $\{15,\,25,\,45\}$        & \num[scientific-notation=false]{15} \\
\texttt{autoencoder\_lr}    & $10^{\mathcal{U}[-5,\,-2]}$ & \num{3.63e-4} \\
\texttt{autoencoder\_wd}    & $10^{\mathcal{U}[-8,\,-4]}$ & \num{5.97e-6} \\
\texttt{reg\_adversary}     & $10^{\mathcal{U}[-2,\,2]}$  & \num{2.31e-2} \\
\texttt{penalty\_adversary} & $10^{\mathcal{U}[-2,\,2]}$  & \num{8.15e-2} \\
\texttt{adversary\_lr}      & $10^{\mathcal{U}[-5,\,-2]}$ & \num{1.7e-4} \\
\texttt{adversary\_wd}      & $10^{\mathcal{U}[-8,\,-4]}$ & \num{3.11e-6} \\
\texttt{adversary\_steps}   & $\{1,\,2,\,3,\,4,\,5\}$     & \num[scientific-notation=false]{5} \\
\bottomrule
\end{tabular}
\end{table}

\subsection{Hyperparameter selection for CPA}
\looseness-1 
For CPA,  hyperparameters are chosen by randomly searching over a predefined parameter space for 100 trials for every dataset, see Tab.~1 of~\citet{lotfollahi2023predicting} for details. To improve the CPA fit on our simulation data and render the comparison with PDAE fairer, we mimic this procedure and run a random search for 100 trials over a subset of these hyperparameters, which we show in~\cref{tab:simdata:cpa:hyperparameters}. We tune all parameters that are neither linked to the architecture nor related to the non-linear dose-response functions. (In the setting of gene perturbation experiments that we emulate here, the dose responses are set to the identity, and thus no doser MLP hyperparameters are needed.) 
After randomly choosing 100 configurations of the parameter space and training a CPA model for 2000 epochs on each, we compute the MMD between predictions and ground truth of 7 ID evaluation environments and select the configuration that minimizes the MMD. The optimal configuration is shown in~\cref{tab:simdata:cpa:hyperparameters}.

\begin{figure}[tbp]
    \centering
    \includegraphics[width=\textwidth]{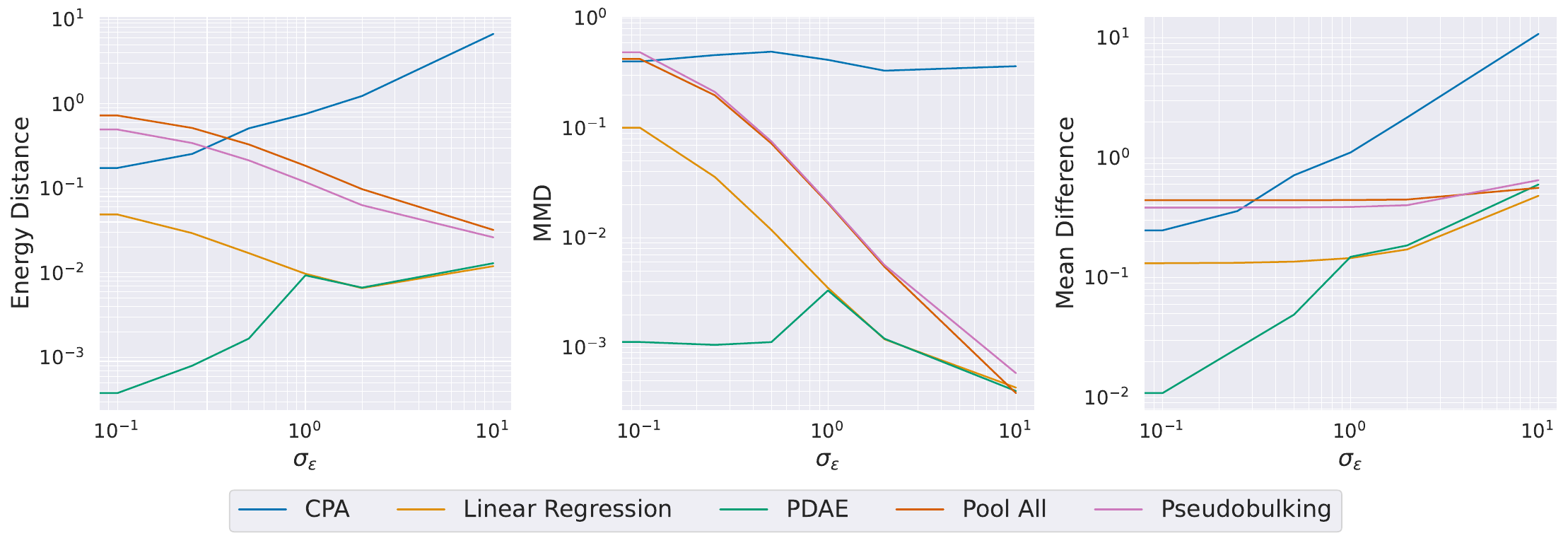}
    \caption{\small \looseness-1 
    \textbf{Evaluation of Robustness Under Increasing Levels of Noise.} We compare the ID test performance of the baseline methods, CPA, and PDAE for higher dimensional observations with additional noise dimensions appended while increasing the level of noise $\sigma_\varepsilon$. 
}
    \label{fig:robustness_experiment}
\end{figure}

\subsection{Robustness experiment with varying signal-to-noise ratio}
\label{app:robustness_experiment}
\looseness-1 
In this experiment, each observation $\xb_{e,i}$ is generated from two signal dimensions $\zb_{e,i}^\text{pert} \in \mathbb{R}^2$ and eight noise dimensions $\varepsilon_{e,i}^\text{pert} \in \mathbb{R}^8$ as input for the nonlinear mixing function, as described at the end of~\cref{sec:experiments}. We evaluate model robustness under increasing levels of noise by varying the standard deviation $\sigma_\varepsilon$ of the noise variable $\varepsilon \sim \mathbb{Q}_\epsilon$ in the data-generating process. The list of values of $\sigma_\epsilon$ considered is $[0,\ 0.1,\ 0.25,\ 0.5,\ 1,\ 2,\ 10]$. We choose the same hyperparameters as described in~\cref{sec:experiments}, with the exception of increasing CPA's bottleneck dimension to 10 and PDAE's decoder model noise dimension to 8. Like this, both models are equipped with a decoder input dimension that fits the dimensionality of observations. On each noise level, we run 8 trials of random search for both methods. For CPA, we tune the set of parameters described in~\cref{sec:experiments}. For PDAE, we tune the reconstruction loss weight $\lambda$, the level of decoder noise and the learning rate. For both methods, the model with the lowest MMD on 7 ID validation environments is selected. Both PDAE and CPA models are then trained for 2000 epochs and evaluated on 7 ID test environments. One random seed is used and shared across data generation, model initialization, training, and evaluation. The results are summarised in~\cref{fig:robustness_experiment}.
As can be seen, PDAE outperforms the other methods across low and moderate levels of $\sigma_\epsilon$ and converges to the performance of Linear Regression at higher levels of noise. 
As $\sigma_\epsilon$ grows large, most signal is lost, and the performance of all methods tends to that of the pooling baseline, as expected.

\begin{figure}[tbp]
  \centering
  \begin{subfigure}{0.33\textwidth}
    \centering
  \begin{tikzpicture}[scale=2]
    \draw[-stealth] (0,0) -- (2.5,0) node[right] {};
    \draw[-stealth] (0,0) -- (0,2.5) node[above] {};
    \draw (1,0) node[below] {1} -- (1,-0.05);
    \draw (2,0) node[below] {2} -- (2,-0.05);
    \draw (0,1) node[left] {1} -- (-0.05,1);
    \draw (0,2) node[left] {2} -- (-0.05,2);
    \fill[orange!30] (0,0) rectangle (2,2);
    \node at (1,1.5) {\Large OOD};
    \fill[blue!30] (0,0) rectangle (1,1);
    \node at (0.5,0.5) {\Large ID};
    \draw (0,0) rectangle (2,2);
    \draw (0,0) rectangle (1,1);
  \end{tikzpicture}
  \caption{}
    \label{fig:simdata:id-ood-regions}
    \end{subfigure}%
  \begin{subfigure}{0.66\textwidth}
        \includegraphics[width=\textwidth]{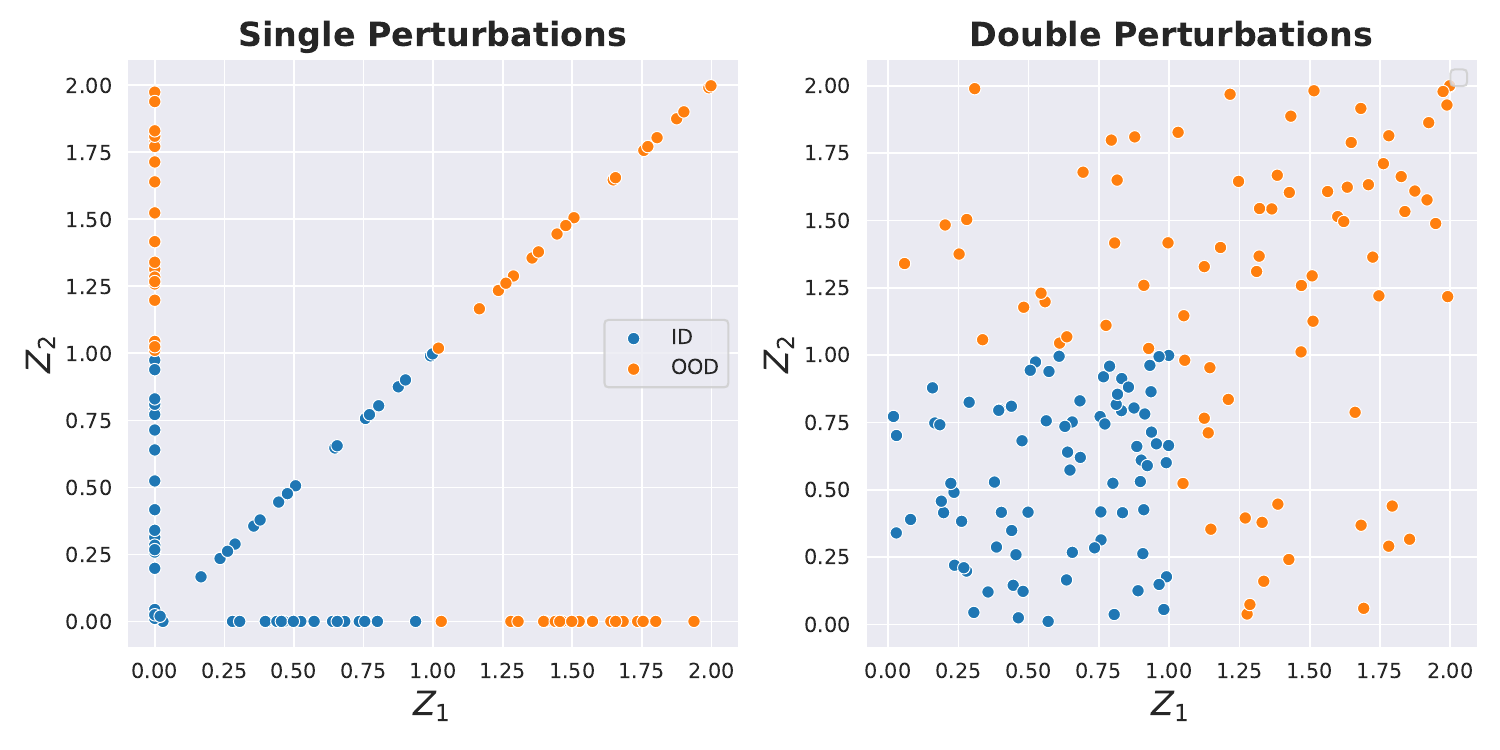}
        \caption{}
          \label{fig:id_ood_meanshifts}
  \end{subfigure}
  \caption{\textbf{(a) Sampling space for ID and OOD test cases.}
    Representation of the ID region (in blue) as defined by the elementary mean shifts in the perturbation matrix and a possible OOD region (in orange) for mean shifts in the Euclidean plane. \textbf{(b) Mean shifts of generated ID and OOD perturbation labels.} Mean shifts from single perturbations (left) are generated from 3 different sampling schemes as detailed above. Mean shifts from double perturbations (right) are generated from 4 different sampling schemes. Overall, $(3+4)*10$ shifts are visualized for ID and OOD. }
\end{figure}

\subsection{Details on validation and test cases} 
\looseness-1 
For every random seed, we randomly generate 14 ID and OOD perturbation labels.
From the fixed training perturbations and the random ID and OOD perturbations, we then generate the two-dimensional latents and corresponding observations. The 14 ID perturbation labels are then %
split into 7 validation labels used for model selection (hyperparameter tuning) and 7 test labels to evaluate performance. 
An example of the resulting ground truth distributions for all synthetic train and test cases is visualized in~\cref{fig:dgp_gt_random_labels}.

\textbf{Random sampling of test cases.}
To make the choice of perturbation labels for validation and testing purposes less arbitrary, we set up a random sampling scheme. We propose two separate schemes for generating ID and OOD perturbation labels since their corresponding mean shifts lie within disjoint regions of the Euclidean plane, as shown in \cref{fig:simdata:id-ood-regions}. The ID region of mean shifts is constrained by our training perturbation labels $\ab_0,...,\ab_3$. Together with the elementary mean shifts $\wb_1,\wb_2,\wb_3$, i.e., the columns of the perturbation matrix $\Wb$%
, the columns of $\Wb\Ab$ constitute the vertices of the unit square as depicted in \cref{fig:simdata:id-ood-regions}. For ease of visualization, we choose the square in the first quadrant with origin at $[0,0]$ and edge length of 2 to constrain the region of OOD mean shifts. 
We now explain how the ID and OOD perturbation labels are sampled such that the corresponding mean shifts adhere to the ID and OOD region, respectively. 

Single perturbation labels like our training labels only contain one non-zero entry.
We therefore sample uniformly from 
\begin{equation}
  \Rcal^\text{ID}_\text{single} = \Big\{(a,0,0):a\in[0,1]\Big\}\cup\Big\{(0,b,0):b\in[0,1]\Big\}\cup\Big\{(0,0,c):c\in[0,1]\Big\}
\end{equation}
For the OOD single perturbation labels, the the non-zero entries lie in $[1,2]$
\begin{equation}
  \Rcal^\text{OOD}_\text{single}=\Big\{(a,0,0):a\in[1,2]\Big\}\cup\Big\{(0,b,0):b\in[1,2]\Big\}\cup\Big\{(0,0,c):c\in[1,2]\Big\}
\end{equation}
Similarly, to generate double perturbations from the ID region, we draw uniformly from
\begin{equation}
    \Rcal^\text{ID}_\text{double}=\Big\{[a,b,0]: a,b \in [0,1]\Big\}\cup \Big\{[a,0,c]: c \in [0,1], a \in [0,1-c]\Big\} \cup   \Big\{[0,b,c]: c \in [0,1], b \in [0,1-c]\Big\}
\end{equation}
where the restriction to $[1-c]$ ensures that shifts remain within the ID region.

To create OOD double perturbation labels that imply admissible mean shifts, we define
\begin{align}
    \Rcal_1 &= \Big\{[a,b,0]: a \in [0,2], b_0 \in [0,1], b=\bm{1}_{\{a \geq 1\}}2b_0 + \bm{1}_{\{a < 1\}}(b_0 + 1)\Big\}
    \\
    \Rcal_2 &= \Big\{[a,b,0]: b \in [0,2], a_0 \in [0,1], a=\bm{1}_{\{b \geq 1\}}2a_0 + \bm{1}_{\{b < 1\}}(a_0 + 1)\Big\}
    \\
    \Rcal_3 &=\Big\{[a,0,c]: c \in [0,2], a_0 \in [0,1], a=a_0(2-\max(1,c)) + \max(1-c, 0)\Big\}
    \\
    \Rcal_4 &= \Big\{[0,b,c]: c \in [0,2], b_0 \in [0,1], b=b_0(2-\max(1,c)) + \max(1-c, 0)\Big\}
\end{align}
and then draw uniformly from
\begin{equation}
    \Rcal^\text{OOD}_\text{double}
    =
    \Rcal_1
    \cup
    \Rcal_2
    \cup 
    \Rcal_3
    \cup   
    \Rcal_4.
\end{equation}

We visualize the effects of 70 generated ID and 70 OOD perturbation labels as mean shifts of base state latents in \cref{fig:id_ood_meanshifts}, i.e., we visualize $ \mub_e^z = \Wb\ab_e$ for $e \in \{1,...,140\}$.

The number of perturbation labels is chosen as a multiple of 7 since we construct 3 single and 4 double perturbation settings for ID and OOD. 

The implied mean shifts for the 4 training, 7 generated ID validation, 7 generated ID test and 14 generated OOD test environments is visualized in \cref{fig:dgp_gt_random_labels}. 

\begin{figure}[tbp]
  \centering
  \includegraphics[width=\textwidth]{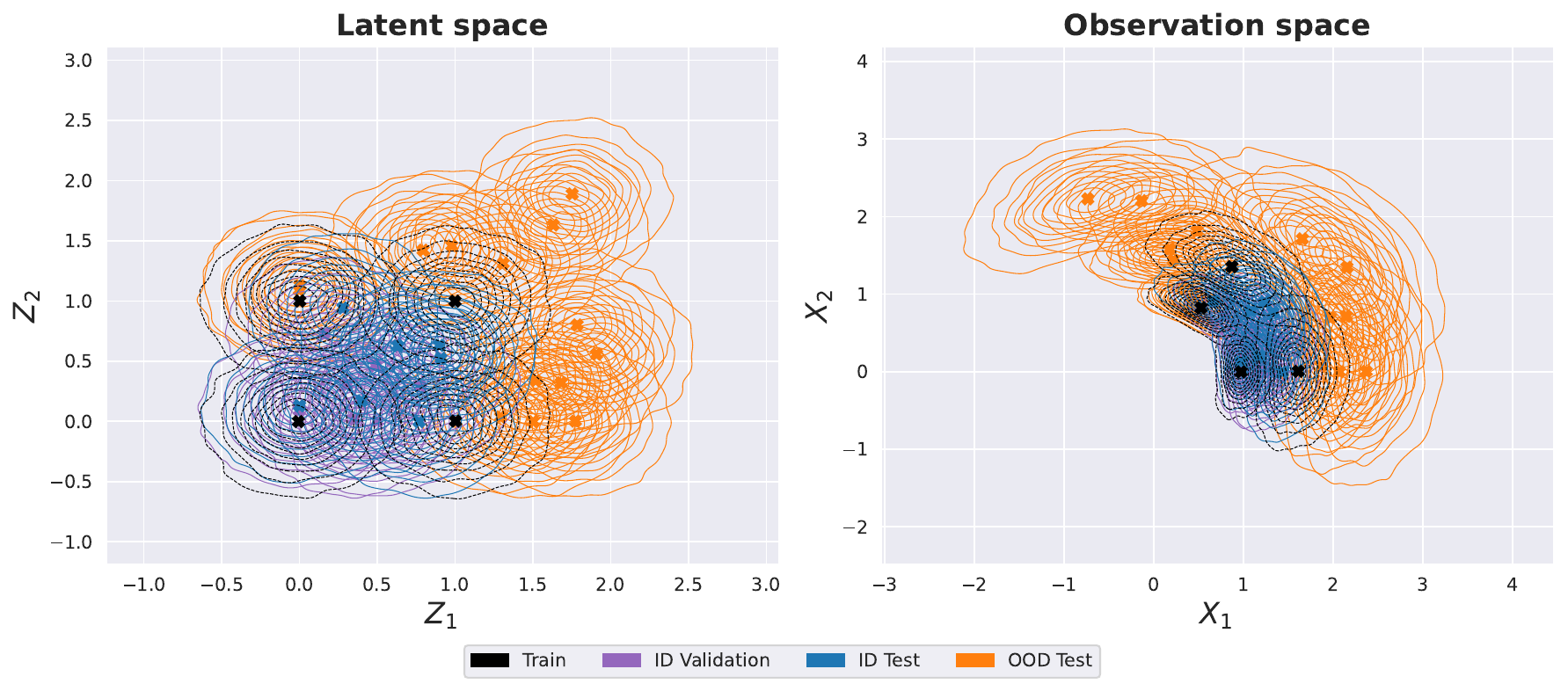}
  \caption{\small \looseness-1 
    \textbf{Fixed train environments and randomly sampled validation and test environments with continous perturbation labels} We visualize the Kernel Density Estimates of the ground truth distribution of the 4 fixed training environments as well as the randomly sampled 14 ID and 14 OOD environments. The ID environments are split into 7 for validation and 7 for testing. The means of the distributions are visualized with crosses. 
  }
  \label{fig:dgp_gt_random_labels}
\end{figure}

\section{Additional background material and related work}
\looseness-1 
Since the problem of interest~(\cref{sec:problem_setting}) involves making distributional predictions for new perturbation conditions, we review some basics of probabilistic forecasting~(\cref{sec:forecasting}) and measuring the similarity between two distributions~(\cref{sec:distributional_similarity}), in our case typically between an empirical distribution and its predicted counterpart. 
We then turn to representation learning with encoder-decoder architectures~(\cref{sec:representation_learning}), in particular a recent approach that also targets distributional reconstruction.
Finally, we cover some prior efforts on perturbation modeling and extrapolation~(\cref{sec:CPA}), which we draw inspiration from.

\subsection{Probabilistic forecasting and scoring rules}
\label{sec:forecasting}
\looseness-1 
Let $\Omega$ be a sample space, 
$\Acal$ a $\sigma$-algebra of subsets of~$\Omega$,
and $\Pcal$ a convex class of probability measures on~$(\Omega, \Acal)$.
A \emph{probabilistic prediction} or \textit{probabilistic forecast} %
is a mapping into $\Pcal$, which outputs predictive distributions $\PP$ over
over outcomes~${x\in\Omega}$.
Probabilistic forecasting can thus be viewed as a distributional generalization point prediction (i.e., deterministic forecasting), which maps directly into~$\Omega$.

To evaluate, compare, or rank different forecasts, it is useful to assign them a numerical score reflecting their quality.
A \textit{scoring rule} is a function $S:\Pcal \times \Omega\to \RR$ that assigns a score $S(\PP, x)$ to forecast~$\PP$ 
if event $x$ materializes, with higher scores corresponding to better forecasts---akin to (negative) loss or cost functions for point predictions.
If $x$ is distributed according to $\QQ$, we denote the \textit{expected score} by $S(\PP,\QQ)=\EE_{x\sim\QQ}[S(\PP,x)]$.
The scoring rule is called \textit{proper} if $S(\QQ,\QQ)\geq S(\PP,\QQ)$ for all $\PP$, and \textit{strictly proper} if equality holds if and only if~$\PP=\QQ$.

\textbf{CRPS.} For continuous scalar random variables (i.e., $\Omega=\RR$), a popular scoring rule is the \textit{continuous ranked probability score}~\citep[CRPS;][]{matheson1976scoring}. When $\Pcal$ is the space of Borel probability measures on~$\RR$ with finite first moment, it is strictly proper and given by:\footnote{The original definition 
by \citet[][]{matheson1976scoring}
is $\CRPS(F,x)=-\int_{-\infty}^\infty (F(y)-1\{y\geq x\})^2\mathrm{d}y$ where $F$ is the CDF of $\PP$, but the simpler form in~\eqref{eq:CRPS} has been shown to be equivalent~\citep[Lemma~2.2]{baringhaus2004new}.}%
\begin{equation}
\label{eq:CRPS}
    \CRPS(\PP,x)=\frac{1}{2}\EE_{X,X'\,\overset{\text{i.i.d.}}{\sim}\,\PP} \abs{X-X'}-\EE_{X\sim\PP} \abs{X-x} \,.
\end{equation}
If $\PP$ is a point mass, the negative CRPS reduces to the absolute error loss function; it can thus be viewed as a generalization thereof to probabilistic forecasts. 

\textbf{Energy score.} 
\Citet{gneiting2007strictly} propose the \textit{energy score} as a multi-variate generalization of the CRPS for vector-valued 
$\xb\in\Omega=\RR^{\dimX}$. 
For $\beta\in(0,2)$, it is defined by
\begin{equation}
    \label{eq:energy_score}
    \ES_\beta(\PP,\xb)=\frac{1}{2}\EE_{\Xb,\Xb'\,\overset{\text{i.i.d.}}{\sim}\,\PP} \norm{\Xb-\Xb'}_2^\beta-\EE_{\Xb\sim\PP} \norm{\Xb-\xb}_2^\beta\,,
\end{equation}
where $\norm{\,\cdot\,}_2$ denotes the Euclidean (L2) norm. $\ES_\beta$ is strictly proper w.r.t.\ $\Pcal_\beta$, the set of Borel probability measures $\PP$ for which  $\EE_{\Xb\sim\PP}\norm{\Xb}_2^\beta$ is finite~\citep{gneiting2007strictly}. 

\subsection{Assessing distributional similarity}
\label{sec:distributional_similarity}
\textbf{Energy distance.} The expected energy score $\ES_\beta(\PP,\QQ)=\EE_{\Yb\sim\QQ}[\ES_\beta(\PP,\Yb)]$ consitutes a measure of similarity between $\PP$ and $\QQ$ and is closely linked to the \textit{energy distance}~\citep{szekely2013energy}:
\begin{align}
    \label{eq:energy_distance}
    \ED_\beta(\PP,\QQ)&=2\EE_{\Xb\sim\PP,\Yb\sim\QQ}\norm{\Xb-\Yb}_2^\beta - \EE_{\Xb,\Xb'\iidsim\PP}\norm{\Xb-\Xb'}_2^\beta - \EE_{\Yb,\Yb'\iidsim\QQ}\norm{\Yb-\Yb'}_2^\beta \\
    &=2\left(\ES_\beta(\QQ,\QQ) - \ES_\beta(\PP,\QQ) \right)\geq 0
\end{align}
with equality if and only if $\PP=\QQ$, since $\ES_\beta$ is a strictly proper scoring rule.

\textbf{Maximum mean discrepancy (MMD).}
Another well-known distance between probability measures that is rooted in kernel methods~\citep{scholkopf2002learning} is the \textit{maximum mean discrepancy}~\citep[MMD;][]{gretton2012kernel}, which for a positive definite kernel $k:\Omega\times\Omega\to\RR$ is given by
\begin{equation}
    \label{eq:MMD}
\MMD^2_k(\PP,\QQ)=
\EE_{X,X'\iidsim\PP}[k(X, X')] 
-2\EE_{X\sim\PP,Y\sim\QQ}[k(X,Y)]
+ \EE_{Y,Y'\iidsim\QQ}[k(Y,Y')]\,.
\end{equation}

\textbf{Energy distance as a special case of MMD.}
As shown by~\citet{sejdinovic2013equivalence}, if $k$ in~\eqref{eq:MMD} is chosen to be the positive-definite \textit{distance kernel}%
\footnote{induced by the negative-definite semi-metric $\rho(\Xb,\Yb)=\norm{\Xb-\Yb}_2^\beta$ on $\Omega=\RR^{\dimX}$ (centered at the origin),}
\begin{equation}
\label{eq:distance_induced_kernel}
k_\beta(\Xb,\Yb)=
\frac{1}{2}\left(\norm{\Xb}_2^\beta+\norm{\Yb}_2^\beta -\norm{\Xb-\Yb}_2^\beta\right)
\end{equation}
then the energy distance is recovered as a special case of MMD,
\begin{equation}
    \label{eq:relation_ED_MMD}
    \ED_\beta(\PP,\QQ)=2 \MMD^2_{k_\beta}(\PP,\QQ)\,.
\end{equation}

\subsection{Representation learning}
\label{sec:representation_learning}
Many modern data sources of interest contain high-dimensional and unstructured  observations~$\xb$,  
such as audio, video, images, or text. 
Representation learning aims to transform such data into a more compact, lower-dimensional set of features~$\zb$ (the representation) which preserves most of the relevant information while making it more easily accessible, e.g., for use in downstream tasks~\citep{bengio2013representation}.\footnote{Representation learning is thus closely related to the classical task of dimension reduction. The former usually refers to nonlinear settings, involves some form of machine learning, and tends to be more focused on usefulness in terms of downstream tasks, rather than, say, explained variance.}
For example, a 
representation of images of multi-object scenes could be a list of the contained objects, along with their size, position, colour, etc.
A key assumption underlying this endeavour is the so-called \textit{manifold hypothesis} which posits that high-dimensional natural data tends to lie near a low-dimensional manifold embedded in the high-dimensional ambient space;
this idea is also at the heart of several (nonlinear) dimension reduction techniques~\citep{cayton2005algorithms,tenenbaum2000global,saul2003think,belkin2001laplacian}.

\textbf{Autoencoder (AE).}
An autoencoder~\citep[AE;][]{hinton2006reducing,rumelhart1986learning} is a pair of functions $(\enc,\dec)$, consisting of an \textit{encoder} $\enc:\RR^{\dimX}\to\RR^{\dimZ}$ mapping observations $\Xb\sim\PP$ to their representation $\Zb:=\enc(\Xb)$, and a \textit{decoder} $\dec:\RR^{\dimZ}\to\RR^{\dimX}$ mapping representations $\Zb$ to their reconstructions in observation space, $\Xbh:=\dec(\Zb)=\dec(\enc(\Xb))$.
Typically, $\dimZ<\dimX$, such that there is a bottleneck and perfect reconstruction is not feasible. 
Autoencoders are usually trained with a (point-wise, mean) reconstruction objective, i.e., the aim is to minimise the mean squared error
\begin{equation}
\label{eq:reconstruction_loss}
 \Lcal_\textsc{ae}(\enc,\dec;\PP)
 :=
 \EE_{\Xb\sim\PP}\norm{\Xb-\Xbh}_2^2 = 
 \EE_{\Xb\sim\PP}\norm{\Xb-\dec(\enc(\Xb))}_2^2\,,
\end{equation}
w.r.t.\ both $\enc$ and $\dec$.
As a result, the optimal decoder $\dec^*_\textsc{ae}$ for any fixed encoder choice $\enc$ is given by the conditional mean
\begin{equation}
    \label{eq:optimal_decoder_AE}
    \dec^*_\textsc{ae}(\zb;\enc, \PP)=\EE_{\Xb\sim\PP}[\Xb|\enc(\Xb)=\zb]\,.
\end{equation}

\textbf{Distributional principal autoencoder (DPA).}
Since the objective of a standard AE is mean reconstruction, 
the distribution of reconstructions $\Xbh$ is typically not the same as the %
distribution 
of $\Xb$
(unless the encoder is invertible, i.e, the compression is lossless and perfect reconstruction is feasible, which is usually not the case in practice).
To address this, \citet{shen2024distributional}~proposed the distributional principal autoencoder (DPA) which targets distributional
(rather than mean) reconstruction. 
A DPA also consists of an encoder-decoder pair $(\enc,\dec)$. However, unlike in a standard AE, the DPA decoder $\dec:\RR^{\dimZ}\times \RR^{d_\epsilon}\to\RR^{\dimX}$ is stochastic in that it  receives an additional noise term $\epsilonb$ as input, which is distributed according to a fixed %
distribution $\QQ_{\epsilonb}$ such as a standard isotropic Gaussian. %
\iftrue
The DPA loss function is constructed such that
for a fixed encoder~$\enc$, the optimal DPA decoder $\dec^*_\textsc{dpa}$ 
maps a 
 given latent embedding~$\zb$ 
to the distribution of 
$\Xb$, given $\enc(\Xb)=\zb$,
 i.e., 
\begin{equation}
    \label{eq:optimal_decoder_DPA}
    \dec^*_\textsc{dpa}(\zb,\epsilonb; \enc)\overset{d}{=}
    \left(\Xb|\enc(\Xb)=\zb\right)\,,
\end{equation}
where $\overset{d}{=}$ denotes equality in distribution.
This means that the decoder evaluated at $\zb$ should
match 
the distribution of all realizations of $\Xb$
that are mapped by the encoder to $\zb$.
\fi
At the same time, the DPA encoder aims to minimise the variability in the distributions in~\eqref{eq:optimal_decoder_DPA} by encoding the first $\dimZ$ ``principal'' components. 
As shown by~\citet{shen2024distributional}, both of these goals are achieved by the following DPA objective,
\begin{equation}
\begin{aligned}
\label{eq:dpa_loss}
 \Lcal_\textsc{dpa}(\enc,\dec;\PP%
 )
 &=
 \EE_{\Xb\sim\PP,\epsilonb\sim\QQ_{\epsilonb} %
 }\norm{\Xb-\dec(\enc(\Xb),\epsilonb)}_2^\beta 
 -\frac{1}{2} \EE_{\Xb\sim\PP,\epsilonb,\epsilonb'\iidsim \QQ_{\epsilonb}%
 }\norm{\dec(\enc(\Xb),\epsilonb)-\dec(\enc(\Xb),\epsilonb')}_2^\beta
 \\
 &= -\EE_{\Xb\sim\PP}\left[
 \ES_\beta\left(
 \dec\left(\enc(\Xb),\,\cdot\,\right)_\#\QQ_{\epsilonb},\Xb\right)
 \right]\,,
\end{aligned}
\end{equation}
where $\dec\left(\zb,\,\cdot\,\right)_\#\QQ_{\epsilonb}$ denotes the pushforward distribution of $\QQ_{\epsilonb}$ through the function $\dec(\zb,\,\cdot\,)$, i.e., the distribution of $\dec(\zb,\epsilonb)$ for a fixed $\zb$ when $\epsilonb\sim\QQ_\epsilonb$.
In other words, a DPA minimizes the negative expected energy score between $\Xb$ %
and the corresponding (stochastic) decoder output, conditional on the %
encoding of $\Xb$.
Due to this conditioning, the DPA objective differs from an energy distance by a ``normalization constant'' which depends on the encoder and encourages capturing principal (i.e., variation-minimizing) components, rather than random latent dimensions.

\subsection{Perturbation modeling}
\label{sec:CPA}

\textbf{Compositional perturbation autoencoder (CPA).}
\citet{lotfollahi2023predicting} propose the compositional perturbation autoencoder (CPA) as a model for compositional extrapolation of perturbation data. 
Specifically, they assume the following model:
\begin{align}
\label{eq:CPA_perturbation_model}
    \zb^\pert&=\zb^\base+\Wb^\pert 
    \begin{pmatrix}
       h_1(a_1)\\
       \dots\\
       h_K(a_K)
    \end{pmatrix}
    +\sum_{j=1}^J\Wb_j^\text{cov}\cbb_j\,,
\end{align}
where $\zb^\base\sim\PP_\Zb$ denotes an unperturbed basal state; the matrix $\Wb^\pert\in\RR^{\dimZ\times K}$ encodes the additive effect of each elementary perturbation; $\{h_k:\RR\to\RR\}_{k\in[K]}$ are unknown, possibly nonlinear dose-response curves; $\{\cbb_j\in\RR^{K_j}\}_{j\in[J]}$ are observed one-hot vectors capturing $J$ additional discrete covariates, such as cell-types or species; and the matrices $\{\Wb_j^\text{cov}\in\RR^{\dimZ\times  K_j}\}_{j\in[J]}$ encode additive covariate-specific effects.
Further, the basal state $\zb^\base$ is assumed to be independent of the perturbation labels $\ab$ and covariates $\Cb=(\cbb_1, \dots, \cbb_J)$.
Observations $\xb$ are then drawn from a %
Gaussian whose mean and variance are determined by the perturbed latent state $\zb^\pert$,
\begin{equation}
\label{eq:CPA_decoder}
    \xb\sim \Ncal\left(\mub\left(\zb^\pert\right),\sigma^2\left(\zb^\pert\right)\Ib\right)\,.
\end{equation}
To fit this model, \citet{lotfollahi2023predicting} employ an adversarial autoencoder~\citep{lample2017fader}. First, an encoder~$\enc$ estimates the basal state 
\begin{equation}
\label{eq:CPA_encoder}
    \zbh^\base=\enc(\xb)\,.
\end{equation}
\looseness-1 The estimated perturbed latent state~$\zbh^\pert$ is then computed according to~\eqref{eq:CPA_perturbation_model} using~\eqref{eq:CPA_encoder} and learnt estimates  $\Wbh^\pert$, $\{\widehat h_k\}$, and $\{\Wbh_l^\text{cov}\}$.
Finally, a (deterministic) decoder $\dec$ uses $\zbh^\pert$ to compute estimates of the mean and variance in~\eqref{eq:CPA_decoder}, i.e., $(\widehat\mub, \widehat\sigma^2)=\dec(\zbh^\pert)$.
All learnable components of the model are trained by minimizing the (Gaussian) negative log-likelihood of the observed data $\Dcal=\{(\xb_i,\ab_i,\Cb_i)\}_{i\in[N]}$.
To encourage the postulated independence of $\zbh^\base$ and $(\ab,\Cb)$, an additional adversarial loss is used, which minimizes the predictability of the latter from the former. 

\textbf{Comparison between PDAE and other biological perturbation models.}
Our model from~\cref{sec:model} is closely related to prior work on modeling combinations of genetic or chemical perturbations. CPA~\citep[][]{lotfollahi2023predicting} 
employs an adversarial autoencoder for estimation, see~\cref{sec:CPA} for details. 
SVAE+~\citep{lopez2023learning} and SAMS-VAE~\citep{bereket2023modeling} instead fit a generative model by maximising the evidence lower bound (ELBO; a lower bound to the marginal likelihood) 
within the variational autoencoder (VAE) framework~\citep{kingma2013auto,rezende2014stochastic}. 
This requires evaluating the likelihood $p(\xb|\zb)$, and thus generally comes with parametric assumptions. 
In contrast, PDAE is based on classical (non-variational) autoencoders and is trained using the energy score, which only requires being able to sample from the stochastic decoder. 
Both SVAE+ and SAMS-VAE additionally assume that the effects of perturbations on the latents are sparse (cf.~\cref{app:algorithm}).
SAMS-VAE and CPA adopt the same compositional perturbation model as~\eqref{eq:perturbation_model}, wherein the effect of perturbations can be decomposed into a sum of effects of elementary perturbations, captured by the linear form $\Wb\ab$.
In contrast, SVAE+ extends work on identifiable VAEs with sparsely changing mechanisms~\citep{lachapelle2022disentanglement,lachapelle2024nonparametric} and models each perturbation condition separately.
Anotther difference to CPA~\citep{lotfollahi2023predicting} and SAMS-VAE~\citep{bereket2023modeling} is that these approaches seek an encoder that maps to the latent basal state, regardless of the domain~$e$, rather than to the perturbed latent state as in PDAE.
Importantly, 
none of these prior perturbation models provide theoretical identifiability or extrapolation guarantees.

\end{document}